\documentclass{article}
\usepackage[final]{colm2026_conference}
\usepackage[utf8]{inputenc}
\usepackage[T1]{fontenc}
\usepackage{hyperref}
\usepackage{url}
\usepackage{booktabs}
\usepackage{amsfonts}
\usepackage{amsmath}
\usepackage{amssymb}
\usepackage{nicefrac}
\usepackage{microtype}
\usepackage{xcolor}
\usepackage{algorithm}
\usepackage{algorithmic}
\usepackage{tikz}
\usepackage{pgfplots}
\pgfplotsset{compat=1.17}
\usepackage{wrapfig}
\usepackage{multirow}
\usepackage{enumitem}
\usepackage{comment}

\usepackage{amsthm}
\newtheorem{proposition}{Proposition}

\newcommand{\vocab}{\mathcal{V}}
\newcommand{\enum}{\mathcal{E}}
\newcommand{\states}{\mathcal{S}}
\newcommand{\bigO}{\mathcal{O}}

\title{Trie Automata for Constrained Decoding over Large Finite Sets}

\author{
  Xingzi Xu \\
  Amazon \\
  \texttt{xingzixu@amazon.com} \\
  \And
  Karim Bouyarmane \\
  Amazon \\
  \texttt{bouykari@amazon.com} \\
}

\begin{document}

\maketitle

\begin{abstract}
Large language models increasingly need to generate structured outputs that conform to predefined schemas, with one common constraint being selection from a finite set of valid strings. Current constrained decoding systems handle this through general-purpose grammar compilation, which becomes prohibitively slow as the number of valid values grows into the thousands, a \emph{cardinality wall}. We introduce the \emph{trie automaton}, a specialized mechanism that exploits finite-set structure (shared prefixes, bounded depth, known cardinality) via Aho-Corasick multi-pattern matching to precompute per-node token masks. The trie achieves 7$\times$ faster per-step valid-token computation (0.65$\mu$s vs.\ 5.8$\mu$s) compared to XGrammar, one of the primary backends in vLLM and SGLang, and 2--6.5$\times$ faster compilation at $K \geq 300$. Because precomputed masks enable a stateless serving path that bypasses the guided decoding pipeline, this advantage compounds in batch serving: end-to-end vLLM throughput reaches 219 req/s vs.\ XGrammar's 7.5 req/s at batch size 256 (29$\times$). This 29$\times$ combines the algorithmic speedup with integration-path savings that only precomputed masks make possible. Across seven tokenizer families (32K--262K vocabulary), the trie maintains sub-100ms compilation up to $K = 10{,}000$ and flat per-step cost regardless of set size, while guaranteeing 100\% output validity.
\end{abstract}

\section{Introduction}

Constrained decoding has become the standard mechanism for guaranteeing that LLM outputs conform to a schema~\citep{willard2023outlines, geng2025jsonschemabench}. By masking invalid tokens at each generation step, it eliminates malformed JSON, hallucinated field names, and invalid values. Major LLM providers now offer it, and open-source engines like Outlines, XGrammar~\citep{dong2024xgrammar}, and SGLang~\citep{zheng2024sglang} have made it accessible to any application.
These systems compile a JSON schema (or grammar) into a general-purpose automaton (a finite-state machine, pushdown automaton, or Earley parser) and use it to mask tokens at each decoding step. This architecture handles arbitrary schemas, including nested objects, recursive structures, and complex regex patterns. However, it applies the same general-purpose compilation pipeline to all constraints, regardless of their actual complexity. A deeply nested recursive JSON schema and a flat list of 1{,}000 tool names both undergo the same compilation pipeline, a fundamental mismatch between general-purpose engines and simple constraints.

This uniformity creates a bottleneck for one of the most common constraints in production: \emph{select one string from a known finite set}. OpenAI's structured outputs impose a 1{,}000 enum limit~\citep{openai2024structured}, Google Gemini fails at approximately 120 enum values~\citep{gemini2024limits}, and Anthropic's 180-second compilation timeout~\citep{anthropic2025structured} implies a similar wall at a few hundred values. These limits are increasingly consequential as LLM applications shift from open-ended generation to structured tool use. In agentic workflows, an LLM must select which tool to invoke from a registry that may contain 500--5{,}000+ APIs~\citep{qin2024toolllm, qu2025mcpzero, wang2025scalable}; as Model Context Protocol (MCP) ecosystems grow and organizations expose internal services as tools, these registries expand rapidly, often exceeding provider enum limits within months of deployment. The same pattern appears in zero-shot classification over label sets like product taxonomies (1{,}500+ categories), ICD-10-CM medical codes (74{,}719 codes in the 2026 CMS release~\citep{cms2026icd10cm}), or legal case types (10{,}000+); in entity linking against knowledge bases~\citep{decao2021genre} with tens of thousands of entries; and in dynamic per-query constraints from retrieval-augmented systems where the valid set changes each query, preventing amortization of compilation costs.
In all these cases, the constraint is a finite union of strings $s_1 | s_2 | \cdots | s_K$. While this is a regular language with no Kleene star, recursion, or nested structure, current systems compile it through the same regex-to-NFA-to-DFA pipeline used for arbitrary grammars, at a cost that grows with both the number of strings and the alphabet size. This creates what we term the \emph{cardinality wall}: a maximum $K$ beyond which constrained decoding becomes impractically slow.

The core insight is that different constraint types deserve different enforcement mechanisms. A finite set of strings has exploitable structure: shared prefixes, finite depth, and known cardinality. We introduce the \emph{trie automaton}, a drop-in replacement for the FSM layer in existing constrained decoding pipelines, specialized for finite-set constraints. It (1) builds a character-level trie~\citep{fredkin1960trie} directly from the set, (2) precomputes vocabulary-aware token masks at each trie node using Aho-Corasick multi-pattern matching~\citep{aho1975efficient} to align BPE tokens with character-level trie paths, and (3) serves masks via $\bigO(1)$ cached lookups at decode time. The central algorithmic challenge is BPE-trie alignment: a single BPE token can span multiple trie nodes, and a vocabulary of $32\text{K}$--$262\text{K}$ tokens must be matched against every node. To our knowledge, this alignment problem has not been addressed in the constrained decoding literature; prior trie-based work~\citep{decao2021genre} sidesteps it by operating at token granularity. We show that it reduces to multi-pattern string matching, solvable in time linear in the trie size rather than quadratic in the vocabulary, enabling sub-100ms compilation up to $K = 10{,}000$.
Figure~\ref{fig:enum_window} illustrates the cardinality wall and how the trie automaton overcomes it, expanding the practical limit from $\sim$1{,}000 to $\sim$100{,}000 values while maintaining 100\% constraint compliance.

%\paragraph{Contributions.} 
This work makes two main contributions:
(1) We introduce the \emph{trie automaton}, a specialized constrained decoding backend for finite-set constraints that combines character-level tries, Aho-Corasick multi-pattern matching, and precomputed token masks to achieve $\bigO(|\text{valid}[s_t]|)$ per-step masking (empirically 10--100 tokens after 3--4 characters of prefix, yielding effectively constant cost versus the $\bigO(V \cdot \ell)$ cost of general FSM approaches). The BPE-trie alignment problem reduces to multi-pattern string matching, yielding 2--6.5$\times$ faster compilation and up to 29$\times$ higher end-to-end throughput in batch serving: 7$\times$ from faster per-step masking, compounded by a simpler serving path that only precomputed masks can use.
(2) We empirically characterize the cardinality wall across seven tokenizer families (32K--262K vocabulary), showing that matching enforcement mechanisms to constraint structure, including the integration path, overcomes scaling bottlenecks. Precomputed masks let the trie bypass the guided decoding pipeline entirely; FSM approaches cannot.
%\end{itemize}

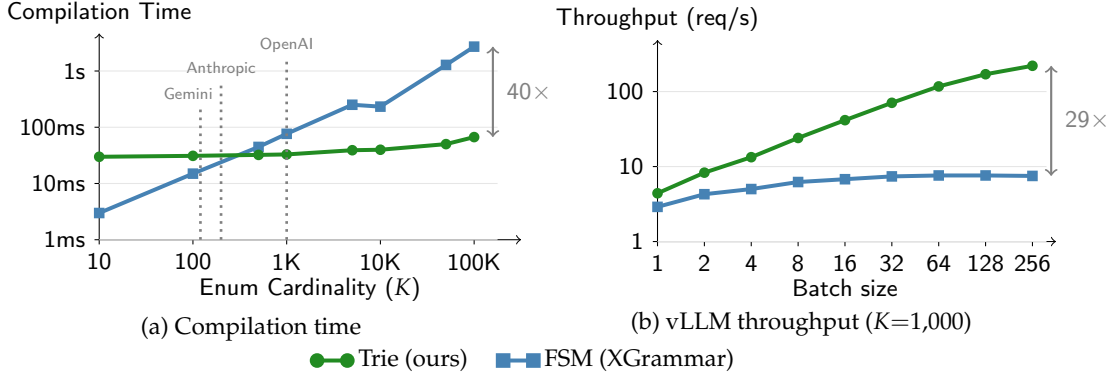
\begin{figure}[hbt!]
\centering
\definecolor{steelblue}{RGB}{70,130,180}
\definecolor{forestgreen}{RGB}{34,139,34}
\begin{minipage}{0.48\textwidth}
\centering
\hfuzz=25pt
\begin{tikzpicture}[scale=0.62]
% Log-log: x = 2*(log10(K)-1), y = 1.2*(log10(ms)+1.5) - 1.8
% So 1ms->0.0, 10ms->1.2, 100ms->2.4, 1s->3.6
\pgfmathsetmacro{\yscale}{1.2}
\pgfmathsetmacro{\yoff}{1.5}

\draw[->] (0,0) -- (9,0);
\node[font=\sffamily\small, anchor=north] at (4.5,-0.6) {Enum Cardinality ($K$)};
\draw[->] (0,0) -- (0,4.4) node[above, font=\sffamily\small] {Compilation Time};

% X axis: log10 scale
\foreach \x/\lab in {0/10, 2/100, 4/1K, 6/10K, 8/100K}
    \draw (\x,0) -- (\x,-0.1) node[below, font={\sffamily\small}] {\lab};

% Y axis: log scale — 1ms, 10ms, 100ms, 1s
\foreach \y/\lab in {0/1ms, 1.2/10ms, 2.4/100ms, 3.6/1s} {
    \draw (0,\y) -- (-0.1,\y) node[left, font={\sffamily\small}] {\lab};
    \draw[gray!20] (0,\y) -- (8.5,\y);
}

% XGrammar: K=10:3ms K=100:15ms K=500:45ms K=1K:76ms K=5K:250ms K=10K:232ms K=50K:1285ms K=100K:2691ms
% y values shifted down by 1.8
\draw[thick, steelblue, line width=1.5pt, mark=square*, mark size=2pt]
    plot coordinates {
        (0, 0.57) (2, 1.41) (3.4, 1.98) (4, 2.26) (5.4, 2.88) (6, 2.84) (7.4, 3.73) (8, 4.12)
    };

% Trie: K=10:30ms K=100:31ms K=500:32ms K=1K:33ms K=5K:39ms K=10K:40ms K=50K:50ms K=100K:67ms
% y values shifted down by 1.8
\draw[thick, forestgreen, line width=1.5pt, mark=*, mark size=2pt]
    plot coordinates {
        (0, 1.77) (2, 1.79) (3.4, 1.81) (4, 1.82) (5.4, 1.91) (6, 1.92) (7.4, 2.04) (8, 2.19)
    };

% Provider limits
\draw[dotted, thick, gray, line width=1pt] (2.16,0) -- (2.16,2.8);
\node[gray, font={\sffamily\tiny}, anchor=south] at (1.9,2.8) {Gemini};
\draw[dotted, thick, gray, line width=1pt] (2.60,0) -- (2.60,3.3);
\node[gray, font={\sffamily\tiny}, anchor=south] at (2.60,3.3) {Anthropic};
\draw[dotted, thick, gray, line width=1pt] (4,0) -- (4,3.8);
\node[gray, font={\sffamily\tiny}, anchor=south] at (4,3.8) {OpenAI};

% Speedup annotation at K=100K
\draw[<->, thick, gray] (8.4,2.19) -- (8.4,4.12);
\node[gray, font={\sffamily\small}, anchor=west] at (8.5,3.15) {40$\times$};

\end{tikzpicture}
\\[-2pt]
{\small (a) Compilation time}
\end{minipage}
\hfill
\begin{minipage}{0.48\textwidth}
\centering
\hfuzz=30pt
\begin{tikzpicture}[scale=0.62]
% Log y: y = 1.6*log10(req/s)
% 1->0, 3->0.76, 10->1.6, 30->2.36, 100->3.2, 300->3.96

\draw[->] (0,0) -- (8.5,0);
\node[font=\sffamily\small, anchor=north] at (4,-0.6) {Batch size};
\draw[->] (0,0) -- (0,4.3) node[above, font=\sffamily\small] {Throughput (req/s)};

% X axis: log2 scale — B=1,2,4,8,16,32,64,128,256
\foreach \x/\lab in {0/1, 1/2, 2/4, 3/8, 4/16, 5/32, 6/64, 7/128, 8/256}
    \draw (\x,0) -- (\x,-0.1) node[below, font={\sffamily\small}] {\lab};

% Y axis: log scale
\foreach \y/\lab in {0/1, 1.6/10, 3.2/100} {
    \draw (0,\y) -- (-0.1,\y) node[left, font={\sffamily\small}] {\lab};
    \draw[gray!20] (0,\y) -- (8.3,\y);
}

% Trie: B=1:4.4, 2:8.3, 4:13.4, 8:23.9, 16:41.8, 32:70.4, 64:116.5, 128:170.5, 256:219.4
% x=log2(B), y=1.6*log10(req/s)
\draw[thick, forestgreen, line width=1.5pt, mark=*, mark size=2pt]
    plot coordinates {(0, 1.03) (1, 1.47) (2, 1.80) (3, 2.21) (4, 2.59) (5, 2.96) (6, 3.31) (7, 3.57) (8, 3.75)};

% XGrammar: B=1:2.9, 2:4.3, 4:5.0, 8:6.2, 16:6.8, 32:7.4, 64:7.6, 128:7.6, 256:7.5
\draw[thick, steelblue, line width=1.5pt, mark=square*, mark size=2pt]
    plot coordinates {(0, 0.74) (1, 1.01) (2, 1.12) (3, 1.27) (4, 1.33) (5, 1.39) (6, 1.41) (7, 1.41) (8, 1.40)};

% Speedup annotation at B=256
\draw[<->, thick, gray] (8.4,1.40) -- (8.4,3.75);
\node[gray, font={\sffamily\small}, anchor=west] at (8.5,2.58) {29$\times$};

\end{tikzpicture}
\\[-2pt]
{\small (b) vLLM throughput ($K{=}1{,}000$)}
\end{minipage}

\vspace{2pt}
% Shared legend
\centering
{\small
\tikz{\draw[forestgreen, thick, line width=1.5pt, mark=*, mark size=2pt] plot coordinates {(0,0) (0.4,0)};} Trie (ours) \quad
\tikz{\draw[steelblue, thick, line width=1.5pt, mark=square*, mark size=2pt] plot coordinates {(0,0) (0.4,0)};} FSM (XGrammar)
}

\caption{(a) Compilation time vs.\ enum cardinality (Qwen3-8B, log-log scale). Dotted lines mark documented provider limits (Gemini ${\sim}$120, Anthropic ${\sim}$200, OpenAI 1{,}000). The trie is flat at 30--67ms; XGrammar crosses the trie at $K{\approx}300$ and reaches 2.7s at $K{=}100\text{K}$. (b) End-to-end vLLM throughput (log scale). The 29$\times$ gap at $B{=}256$ combines 7$\times$ per-step algorithmic advantage with integration-path savings from precomputed masks (Section~\ref{sec:experiments}).}
\label{fig:enum_window}
\end{figure}

\section{Background and Problem Formulation}
\label{sec:background}

Constrained decoding restricts the model's output distribution at each step $t$ to tokens that can lead to valid completions. Given a vocabulary $\vocab$ of size $V$ and a regular language $\mathcal{L}$ defined by a schema, the constrained distribution is:
\begin{equation}
    p_c(y_t \mid \mathbf{y}_{<t}) = \frac{p(y_t \mid \mathbf{y}_{<t}) \cdot \mathbf{1}[y_t \in \mathcal{A}(\mathbf{y}_{<t})]}{Z(\mathbf{y}_{<t})}
\end{equation}
where $\mathcal{A}(\mathbf{y}_{<t}) \subseteq \vocab$ is the set of allowed tokens given the prefix, computed by maintaining an FSM state $s_t = \delta^*(s_0, \text{chars}(\mathbf{y}_{<t}))$ (where $\delta^*$ extends the character-level transition function to token sequences and $s_0$ is the start state) and checking valid transitions; equivalently, $\mathcal{A}(\mathbf{y}_{<t}) = \mathcal{A}(s_t)$ depends only on the current state $s_t$, not the full prefix.
An enum constraint $\enum = \{e_1, e_2, \ldots, e_K\}$ defines the regular language $\mathcal{L}_\enum = \{e_1\} \cup \{e_2\} \cup \cdots \cup \{e_K\}$. Let $L_{\max} = \max_i |e_i|$ be the maximum string length, $\ell$ the maximum token length in characters, and $|\Sigma|$ the alphabet size (256 for byte-level BPE tokenizers, which we refer to as ``characters'' throughout for readability; multi-byte UTF-8 sequences and emoji are handled naturally since both the trie and BPE tokenizers operate at byte granularity). The standard approach converts the enum to a regular expression $e_1 | e_2 | \cdots | e_K$ and compiles it into a deterministic FSM. This creates problematic scaling: the DFA has $\bigO(K \cdot L_{\max})$ states in the worst case, per-step masking costs $\bigO(V \cdot \ell)$ (checking each token's character sequence against the FSM), and compilation costs $\bigO(K \cdot L_{\max} \cdot |\Sigma|)$, creating the cardinality wall observed in practice (detailed analysis in Appendices~\ref{app:background_detailed} and~\ref{app:theory}).

To illustrate concretely, consider an agentic system routing requests to the correct tool from a registry of 2{,}000 APIs~\citep{qu2025mcpzero, wang2025scalable}. FSM compilation requires processing 15.4 million character-level transitions (25--50 seconds), while per-step masking requires ${\sim}$1.3 million effective FSM operations per tool selection when accounting for cache effects (Appendix~\ref{app:concrete_example}), making the system unusable for interactive workflows that demand sub-second tool dispatch.

\section{Related Work}
\label{sec:related}

\paragraph{Constrained decoding.} Early approaches enforced specific constraint types: lexically constrained beam search~\citep{post2018fast, anderson2017guided}, predicate logic constraints~\citep{lu2021neurologic, lu2022neurologic}, and incremental parsing for code generation~\citep{scholak2021picard}. Outlines~\citep{willard2023outlines} introduced the dominant paradigm of compiling JSON schemas into FSMs for token masking, extended to context-free grammars by \citet{geng2023grammar} and formalized by \citet{koo2024automata}. LMQL~\citep{beurer2023lmql} embedded constraints into a query language. Subsequent work optimized within this paradigm: SGLang~\citep{zheng2024sglang} introduced jump-forward decoding, XGrammar~\citep{dong2024xgrammar} optimized vocabulary partitioning, and SynCode~\citep{ugare2024syncode} and \citet{wang2025leverage} continued this trajectory. LLGuidance~\citep{geng2025jsonschemabench} takes a different approach: an Earley parser with lazy automaton construction that avoids upfront DFA compilation. On Qwen3-8B (151K vocabulary), LLGuidance compiles enum schemas in 0.6--24ms for $K = 10$--$10{,}000$, far faster than XGrammar's 5--695ms. However, its per-step cost remains $\bigO(V)$: we measure 73--141$\mu$s per mask computation, compared to our trie's 0.65$\mu$s (110--215$\times$ slower; Appendix~\ref{app:llguidance_bench}). The two are complementary: LLGuidance excels at schema diversity with negligible startup, while the trie exploits finite-set structure for per-step speedups that compound in batch serving. Despite these advances, persistent enum limits across major providers indicate that no current system adequately addresses the cardinality wall.

\paragraph{Trie-based generation.} GENRE~\citep{decao2021genre} demonstrated trie-constrained generation for entity linking with a fine-tuned seq2seq model, building a \emph{token-level} trie whose nodes are pre-tokenized token IDs, which sidesteps the vocabulary-trie alignment problem but shares prefixes only at token boundaries. Our trie automaton instead builds a \emph{character-level} trie that maximizes prefix sharing regardless of tokenizer, solving the BPE alignment problem via Aho-Corasick multi-pattern matching~\citep{aho1975efficient} to determine which BPE tokens from a vocabulary of $32\text{K}$--$262\text{K}$ entries are valid continuations at each node (Section~\ref{sec:ac}). We compare the two constructions directly in Section~\ref{sec:ac}: they cross over at $K \approx 1{,}000$, exactly the cardinality-wall regime, and the character-level trie is tokenization-agnostic where GENRE's is tied to one fixed tokenization. Concurrent work by \citet{su2026static} flattens token-level prefix trees into CSR sparse matrices for vectorized TPU/GPU execution; like GENRE, it operates on small semantic ID vocabularies ($|\mathcal{V}| \approx 2{,}048$) where the BPE alignment problem does not arise. Tries also structure LLM decoding pipelines beyond finite-set enforcement: \citet{chan2025beam} share KV-cache across beams with common prefixes to accelerate beam search, and \citet{liu2025zeroshot} apply trie-based contextual biasing for rare words in ASR via soft shallow-fusion rewards. Both index emerging hypotheses or bias scores rather than a fixed enum, and are complementary to the hard-constraint masking we study. Finally, \citet{cognetta2025tokenization} cast tokenization itself as finite-state transduction and give a polynomial-time framework for enforcing canonical tokenization, complementary to our character-level constraint and layerable on top of it without changing the trie (Appendix~\ref{app:canonical}).

\section{Methods}
\label{sec:methods}

\subsection{Character-Level Trie with Precomputed Masks}
\label{sec:trie}

Given an enum $\enum = \{e_1, \ldots, e_K\}$, we build a character-level trie $\mathcal{T}$ by inserting each string character-by-character. The trie~\citep{fredkin1960trie} merges shared prefixes: if many tool names start with \texttt{get\_}, they share a single prefix path rather than duplicating it for each value. This reduces the number of nodes from $\bigO(K \cdot L_{\max})$ (as in the equivalent DFA) to $\bigO(N_{\text{chars}})$, where $N_{\text{chars}} = \sum_{i=1}^K |e_i|$ is the total character count.
For each trie node $n$, we precompute a set $\text{valid}[n] \subseteq \vocab$ of vocabulary tokens that are valid continuations from $n$. At decode time, masking reduces to a direct lookup: $\mathcal{A}(s_t) = \text{valid}[s_t]$, transforming per-step cost from $\bigO(V \cdot \ell)$ (scanning all tokens against the FSM) to $\bigO(|\text{valid}[s_t]|)$ (iterating the precomputed set).

The challenge is computing $\text{valid}[n]$ efficiently. A token $v$ is valid at node $n$ if its character sequence traces a path starting at $n$ that stays within the trie, but BPE tokens are multi-character, so a single token can traverse several trie edges. For example, given enum values \texttt{medical\_billing} and \texttt{medical\_coding}, the token \texttt{``medical''} (7 characters) is valid at the root because it traces the path through nodes \texttt{m}$\to$\texttt{e}$\to \cdots \to$\texttt{l}, while the token \texttt{``\_bill''} is valid at the node after \texttt{l} because it continues along the \texttt{\_billing} branch (Figure~\ref{fig:trie_alignment}). A naive approach checks every token at every node by simulating the character walk, costing $\bigO(N_{\text{chars}} \cdot V \cdot \ell)$, which is prohibitive for vocabularies of $32\text{K}$--$262\text{K}$ tokens.

\begin{figure}[t]
\centering
\begin{tikzpicture}[
    scale=0.95,
    trienode/.style={circle, draw=black!60, fill=white, minimum size=5mm, inner sep=0pt, font=\scriptsize\sffamily},
    leafnode/.style={circle, draw=black!60, fill=black!8, minimum size=5mm, inner sep=0pt, font=\scriptsize\sffamily},
    elabel/.style={font=\scriptsize\sffamily, midway, above=-1pt},
]

% Colors
\definecolor{cblue}{HTML}{4878D0}
\definecolor{corange}{HTML}{EE854A}
\definecolor{cpurple}{HTML}{956CB4}

% === HORIZONTAL TRIE ===
\node[trienode] (root) at (0, 0) {$\circ$};
\node[trienode] (medic) at (2, 0) {};
\node[trienode] (a) at (3.5, 0) {};
\node[trienode] (l_) at (5, 0) {};
\node[leafnode] (billing) at (7, 0.7) {};
\node[leafnode] (coding) at (7, -0.7) {};
\node[leafnode] (ation) at (4.0, -1.4) {};

% Edges
\draw[->, >=stealth, thick, black!50] (root) -- (medic) node[elabel] {\texttt{medic}};
\draw[->, >=stealth, thick, black!50] (medic) -- (a) node[elabel] {\texttt{a}};
\draw[->, >=stealth, thick, black!50] (a) -- (l_) node[elabel] {\texttt{l\_}};
\draw[->, >=stealth, thick, black!50] (l_) -- (billing) node[elabel, above=3pt] {\texttt{billing}};
\draw[->, >=stealth, thick, black!50] (l_) -- (coding) node[elabel, below] {\texttt{coding}};
\draw[->, >=stealth, thick, black!50] (medic) -- (ation) node[elabel, right=2pt] {\texttt{ation}};

% === BPE TOKEN SPANS ===
% Token: 'medic' (above, root → medic)
\draw[cblue, line width=2.5pt, rounded corners=1pt] (0, 0.45) -- (0, 0.65) -- (2, 0.65) -- (2, 0.45);
\node[cblue, font=\scriptsize\sffamily\bfseries, above] at (1, 0.65) {\texttt{medic}};

% Token: 'al_' (above, a → l_) — raised higher to avoid overlap with billing
\draw[corange, line width=2.5pt, rounded corners=1pt] (3.5, 0.45) -- (3.5, 0.65) -- (5, 0.65) -- (5, 0.45);
\node[corange, font=\scriptsize\sffamily\bfseries, above] at (4.25, 0.65) {\texttt{al\_}};

% Token: 'ation' (below, bracket under the ation branch)
\draw[cpurple, line width=2.5pt, rounded corners=1pt] (4.0, -1.85) -- (4.0, -2.05) -- (2, -2.05) -- (2, -1.85);
\node[cpurple, font=\scriptsize\sffamily\bfseries, below=-1pt] at (3.0, -2.05) {\texttt{ation}};
% dashed connectors — short, only bridging gap
\draw[cpurple, line width=0.6pt, dashed] (2, -0.35) -- (2, -1.85);

% === VALID TABLE (right side) ===
\node[draw=black!40, fill=white, rounded corners=3pt, line width=0.6pt,
      text width=4.0cm, align=left, font=\scriptsize\sffamily,
      anchor=north west] (vtable) at (8.0, 0.8) {
    \textbf{Precomputed masks}\\[2pt]
    valid[root] = \{\textcolor{cblue}{\textbf{medic}}, m, \ldots\}\\[1pt]
    valid[medic] = \{\textcolor{corange}{\textbf{al\_}}, \textcolor{cpurple}{\textbf{ation}}\}\\[1pt]
    valid[l\_] = \{billing, coding\}
};

\end{tikzpicture}
\caption{The BPE-trie alignment problem. A character-level trie encodes three enum values with shared prefix structure. BPE tokens (colored brackets) span multiple trie edges: \texttt{medic} traverses 5 character nodes from the root. At each node, we precompute the set of valid BPE tokens (right), enabling precomputed mask lookups at decode time.}
\label{fig:trie_alignment}
\end{figure}
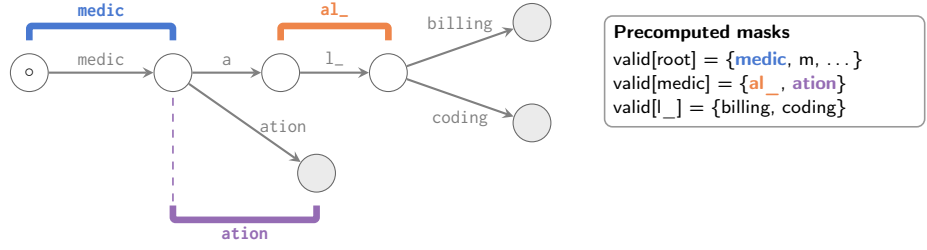

\subsection{Efficient Mask Precomputation via Multi-Pattern Matching}
\label{sec:ac}

We solve the precomputation problem by reducing it to \emph{multi-pattern string matching}: given a set of patterns (the vocabulary token strings) and a text (each root-to-leaf path in the trie), find all positions where any pattern occurs. The Aho-Corasick (AC) algorithm~\citep{aho1975efficient} solves this classic problem by building a finite automaton from the pattern set in $\bigO(V \cdot \ell)$ time, where $\ell$ is the maximum token length in characters, then processing any input text in a single linear pass, reporting all pattern matches as they are encountered. The processing cost is $\bigO(L + m)$ for a text of length $L$ with $m$ reported matches, independent of the number of patterns $V$. Since at each of the $L$ character positions at most $\ell$ tokens of different lengths can match, $m \leq L \cdot \ell$, so scanning a trie path of length $L$ costs $\bigO(L \cdot \ell)$ rather than $\bigO(L \cdot V \cdot \ell)$.
We apply this as follows:
\begin{enumerate}[nosep,leftmargin=*]
\item \textbf{Build AC automaton from vocabulary.} Insert each vocabulary token's character string as a pattern. This constructs the automaton in $\bigO(V \cdot \ell)$ time.
\item \textbf{Traverse the trie with the automaton.} Before any decoding begins, we perform a depth-first traversal of the trie, maintaining the AC automaton state across edges. The AC state is a single integer (current node index), so save/restore at branch points is trivial (push/pop one integer per DFS stack frame), and each edge is processed exactly once. At each trie node $n$, the automaton identifies which vocabulary tokens have a character string that starts at $n$ and follows a valid path in the trie. A matched token $v$ is added to $\text{valid}[n]$ if its characters trace a path from $n$ that either reaches a leaf (completing an enum value) or lands on an internal node (allowing continuation by subsequent tokens). Tokens that would ``overshoot'' a leaf, i.e., whose characters extend beyond the end of an enum value, are excluded. When an enum value is a proper prefix of another (e.g., \texttt{"get"} and \texttt{"get\_user"}), the trie marks the shorter value's terminal node as both a leaf (where EOS is valid) and an internal node (where continuation tokens are valid), ensuring both values are reachable. This is done once per schema and cached.
\end{enumerate}

This reduces precomputation from $\bigO(N_{\text{chars}} \cdot V \cdot \ell)$ to $\bigO((N_{\text{chars}} + V) \cdot \ell)$, a factor of $V$ improvement. In practice, the AC automaton construction ($\bigO(V \cdot \ell)$) dominates the total compilation cost because $V \gg N_{\text{chars}}$ for typical enum sizes: with $V = 151\text{K}$ and $\ell = 4$, the AC construction requires ${\sim}600\text{K}$ operations regardless of $K$, while the trie traversal adds only $N_{\text{chars}} \cdot \ell$ operations (e.g., $240\text{K}$ at $K = 2{,}000$). This explains the empirically observed near-flat compilation time (30--40ms for Qwen3-8B across $K = 10$--$10{,}000$; 67ms at $K = 100{,}000$): the cost is dominated by the $K$-independent AC construction over the vocabulary. Since the AC automaton depends only on the tokenizer (not the enum set), it can be built once per tokenizer and reused across all enum schemas, reducing per-schema compilation to just the trie traversal: $\bigO(N_{\text{chars}} \cdot \ell)$.
At decode time, masking is a direct lookup: $\mathcal{A}(s_t) = \text{valid}[s_t]$, costing $\bigO(|\text{valid}[s_t]|)$ instead of $\bigO(V \cdot \ell)$. Note that applying the mask to the logits vector is $\bigO(V)$ regardless of method (setting invalid logits to $-\infty$); the savings is in \emph{computing} which tokens are valid, not in applying the result. Since $|\text{valid}[s_t]|$ shrinks exponentially with trie depth (after 3--4 characters of prefix, typically 10--100 tokens remain valid), the lookup cost is effectively constant at 0.65$\mu$s regardless of $K$ or vocabulary size. While the worst case is $\bigO(V)$ (at the root node), the rapid prefix-driven shrinkage means the practical per-step cost is orders of magnitude below FSM approaches, which must scan all $V$ tokens regardless of constraint structure.

Table~\ref{tab:complexity_comparison} summarizes both effects against the FSM. The compilation advantage comes from avoiding the $|\Sigma|$ factor (the FSM fills transition entries for all 256 byte values at each state, while the trie processes only characters that appear); the per-step advantage comes from replacing a full vocabulary scan with a cached lookup whose working set ($<$1~KB) fits in L1 cache, versus the FSM's ${\sim}$47~MB transition table that overflows L2 (Appendix~\ref{app:theory}).

\begin{table}[hbt!]
\centering
\caption{Complexity comparison between FSM and trie automaton. Concrete values use $K = 2{,}000$, $L_{\max} = 30$, $|\Sigma| = 256$, $V = 32{,}000$, $\ell = 4$.}
\label{tab:complexity_comparison}
\small
\begin{tabular}{lccc}
\toprule
& \textbf{FSM} & \textbf{Trie (Aho-Corasick)} & \textbf{Ratio} \\
\midrule
Compilation & $\bigO(N_{\text{chars}} \cdot |\Sigma| \cdot \log N_{\text{chars}})$ & $\bigO((N_{\text{chars}} + V) \cdot \ell)$ & $\sim\!\frac{|\Sigma| \cdot \log N_{\text{chars}}}{\ell}$ \\
\quad \emph{concrete (subset only)} & $\sim$15.4M ops & $\sim$370K ops & $\sim$40$\times$ \\
\addlinespace
Per-step mask & $\bigO(V \cdot \ell)$ & $\bigO(|\text{valid}[s_t]|)$ & $\sim\!\frac{V \cdot \ell}{|\text{valid}|}$ \\
\quad \emph{concrete} & $\sim$128K lookups & $\sim$50--500 lookups & $\sim$100$\times$ \\
\addlinespace
Working set & $|\states| \cdot |\Sigma| \cdot 4$~B & $|\text{valid}[s_t]| \cdot 4$~B & \\
\quad \emph{concrete} & $\sim$47~MB & $<$1~KB & fits L1 \\
\bottomrule
\end{tabular}
\end{table}

The trie automaton is not an approximation: it produces \emph{identical outputs} to FSM-based constrained decoding, which is the paper's central formal guarantee and the basis for the 100\% validity we report.

\begin{proposition}[Output equivalence]
\label{prop:correctness}
Let $\vocab_{\textup{dec}} \subseteq \vocab$ be the vocabulary tokens with valid character decompositions (excluding special tokens such as \textup{\texttt{<pad>}}, \textup{\texttt{<unk>}}). For any enum $\enum$ and any prefix $\mathbf{y}_{<t}$, the constrained distributions of the FSM and the trie automaton are identical, $p_c^{\textup{FSM}}(y_t \mid \mathbf{y}_{<t}) = p_c^{\textup{trie}}(y_t \mid \mathbf{y}_{<t})$ for all $y_t \in \vocab_{\textup{dec}}$. Consequently, greedy decoding and fixed-seed sampling produce identical outputs under both methods.
\end{proposition}

\emph{Proof sketch.} The valid-token sets suffice, since $p_c(y_t \mid \mathbf{y}_{<t}) \propto p(y_t \mid \mathbf{y}_{<t}) \cdot \mathbf{1}[y_t \in \mathcal{A}(s_t)]$. By Myhill-Nerode, the equivalence classes of $\mathcal{L}_\enum$ are exactly the distinct enum-value prefixes (plus a dead state), so the trie is isomorphic to the minimal DFA for $\mathcal{L}_\enum$ and their transition functions agree. Both methods admit $v$ at $s_t$ iff its characters trace a live path from $s_t$ that can still reach an accept state, so the sets coincide. The equivalence is over the decodable vocabulary $\vocab_{\text{dec}}$; the two can differ only in the trailing EOS, which does not change the decoded string. The full proof and EOS handling are in Appendix~\ref{app:theory}.

\paragraph{Worked example.} Consider $\enum = \{\texttt{medical\_billing}, \texttt{medical\_coding}, \texttt{medical\_records}\}$ and a toy vocabulary $\{\texttt{medical}, \texttt{\_bill}, \texttt{\_cod}, \texttt{\_rec}, \texttt{ing}, \ldots\}$. The trie merges the shared \texttt{medical\_} prefix into one path of nine nodes and then branches three ways; call the branch node $n$ (reached after \texttt{medical\_}) and the node after \texttt{medical} its parent $n'$. Traversing the trie while carrying the AC state, the automaton reports \texttt{medical} ending at $n'$, so \texttt{medical} enters $\text{valid}[\text{root}]$; descending each branch reports \texttt{\_bill}, \texttt{\_cod}, and \texttt{\_rec}, all starting at $n'$, so all three enter $\text{valid}[n']$. Because the DFS restores the AC state to its value at $n'$ before each branch, the three suffix tokens are attributed to the same node rather than leaking across branches. The precomputed masks are therefore $\text{valid}[\text{root}] = \{\texttt{medical}\}$ and $\text{valid}[n'] = \{\texttt{\_bill}, \texttt{\_cod}, \texttt{\_rec}\}$, with a single continuation token at each deeper node. At decode time the model emits \texttt{medical} from the root, then chooses one of \texttt{\_bill}/\texttt{\_cod}/\texttt{\_rec} to commit to a value, and follows the unique remaining path to the leaf, so the vocabulary scan for the shared prefix is paid once at compile time rather than at every step. A full trace with failure links and the mask at every node is given in Appendix~\ref{app:worked_example}.

\paragraph{Why character-level rather than token-level?} GENRE~\citep{decao2021genre} avoids the BPE-alignment problem by building the trie over \emph{token IDs}: each enum value is pre-tokenized once and the valid tokens at any node are simply its children, requiring no vocabulary scan, no AC automaton, and no mask precomputation. The character-level construction is more expensive per schema, and the reason to prefer it is how the two scale with $K$. A token-level trie shares prefixes only at token boundaries: two values sharing a 10-character prefix may share a single token-level node if tokenized differently, so its size, and hence its compilation cost, grows roughly linearly with $K$. The character-level trie merges every shared character, and its dominant cost is the $K$-independent AC build, which amortizes across all shared prefixes. Table~\ref{tab:genre_main} shows the consequence: the token trie is faster below $K \approx 1{,}000$ (it simply does less work), but the two cross over there and the char trie is $7\times$ faster by $K = 10{,}000$. This crossover sits exactly at the cardinality wall we target, namely provider limits of 120 to 1{,}000 and dynamic per-query sets that force recompilation, so the token-level trie scales linearly \emph{into} the wall while the char-level trie clears it. Per-step masking is sub-microsecond for both, both produce byte-identical outputs (Proposition~\ref{prop:correctness}), and a Python token-trie reproduces the same crossover, confirming the effect is algorithmic, not a Rust-vs-Python artifact (Appendix~\ref{app:genre_comparison}). The character-level construction is also a drop-in backend for any tokenizer and is tokenization-agnostic (Appendix~\ref{app:canonical}), whereas GENRE's trie is tied to one fixed tokenization.

\begin{table}[hbt!]
\centering
\caption{Compilation: character-level trie + Aho-Corasick vs.\ GENRE-style token-level trie (Qwen3-8B, 151K vocab, synthetic tools). The token trie grows with $K$; the char trie pays a fixed ${\sim}30$ms AC-build overhead and stays flat. Both yield byte-identical outputs (Appendix~\ref{app:genre_comparison}).}
\label{tab:genre_main}
\small
\begin{tabular}{rrrrr}
\toprule
$K$ & Token trie (ms) & Char trie + AC (ms) & Char nodes & Token nodes \\
\midrule
100 & \textbf{4.5} & 31 & 1{,}380 & 456 \\
1{,}000 & 37 & \textbf{35} & 8{,}462 & 4{,}269 \\
5{,}000 & 183 & \textbf{43} & 24{,}837 & 18{,}925 \\
10{,}000 & 369 & \textbf{52} & 42{,}460 & 36{,}515 \\
\bottomrule
\end{tabular}
\end{table}

\subsection{Extensions and Integration}
\label{sec:extensions}

The trie automaton is designed as a \emph{drop-in component} for mixed schemas: the system routes finite-set constraints to the trie while delegating structural constraints to the standard FSM/PDA backend, requiring no changes to user schemas. For $K > 50{,}000$, we sketch two preliminary extensions in Appendices~\ref{app:hierarchical} and~\ref{app:speculative} for future directions: hierarchical schema rewriting ($\bigO(\sqrt{K})$ effective per-step cardinality) and speculative short-circuiting.

\section{Experiments}
\label{sec:experiments}

We evaluate the trie automaton on latency, compilation time, and accuracy/validity across various open-sourced model series on NVIDIA A100 GPUs. The primary comparison is against xgrammar~\citep{dong2024xgrammar}, the backend used by vLLM and SGLang. The trie automaton is implemented in Rust with Python bindings via PyO3; XGrammar is implemented in C++ with Python bindings. Full experimental details are in Appendix~\ref{app:experimental_setup}.

\subsection{Latency and Scalability}

Table~\ref{tab:performance} shows the performance breakdown across $K \in \{10, 100, 1{,}000, 10{,}000\}$ for XGrammar~\citep{dong2024xgrammar}, LLGuidance~\citep{geng2025jsonschemabench}, and our trie. All benchmarks run on NVIDIA A100 GPUs (80GB) with AMD EPYC 7R32 CPUs (96 cores).

Three approaches occupy distinct points in the compilation-vs-masking tradeoff (Table~\ref{tab:llguidance}). LLGuidance achieves near-zero compilation (0.6--24ms) but per-step masking costs 73--141$\mu$s. XGrammar balances both (3--239ms compilation, 5--10$\mu$s masking). The trie minimizes per-step cost (0.65$\mu$s) through precomputation, with nearly-flat compilation (30--40ms). For one-shot dynamic schemas where compilation dominates (e.g., retrieval-augmented settings with per-query enum sets at $K < 500$), LLGuidance's 1--3ms compilation may be preferable despite higher per-step cost; the trie's advantage is decisive when per-step cost dominates, particularly in batch serving (Appendix~\ref{app:guidance}).

The trie's advantage is decisive in batch serving, where the GPU forward pass is shared across $B$ requests but masking runs per-request on CPU. At $B = 128$, LLGuidance masking (3.7ms) consumes 37\% of the GPU forward pass, becoming the throughput bottleneck; XGrammar's 783$\mu$s is 7.8\%; the trie's 10$\mu$s is negligible (0.1\%) (per-batch-step masking costs across $B$ in Appendix~\ref{app:batch_throughput}, Table~\ref{tab:batch_throughput}). We validate this with end-to-end vLLM throughput at $K = 1{,}000$ (Table~\ref{tab:vllm_throughput}): at $B = 256$, the trie achieves 219 req/s vs.\ XGrammar's 7.5 req/s (29$\times$). The trie also exceeds unconstrained throughput (219 vs.\ 104 req/s) because constrained decoding terminates at trie leaf nodes (3.2 tokens/request vs.\ 8.7 unconstrained), reducing GPU forward passes. Both constrained methods generate the same number of tokens, so the trie-vs-XGrammar ratio isolates the masking and integration differences.

The 29$\times$ gap compounds two effects. First, the per-step algorithmic advantage: precomputed mask lookup costs 0.65$\mu$s vs.\ XGrammar's 5.9$\mu$s dynamic computation (${\sim}$7$\times$; Table~\ref{tab:performance}). Second, the integration path: because the trie's masks are precomputed, it integrates as a stateless \texttt{LogitsProcessor} that returns a cached bitmask per step. XGrammar does not currently support this path; its architecture requires dynamic mask computation via vLLM's guided decoding pipeline with per-request grammar compilation, sequential FSM state management, and scheduling overhead. In principle, XGrammar \emph{could} precompute and cache per-state masks for enum constraints, but doing so would effectively reconstruct the trie: the minimal DFA for a finite set is isomorphic to the trie (Proposition~\ref{prop:correctness}), so caching its per-state masks yields the same data structure. The trie is thus the natural endpoint of optimizing FSM-based decoding for finite sets.

%\paragraph{Deployment implications.} The batch serving results have direct consequences for GPU utilization in production. At $B = 128$ with XGrammar, CPU masking (783$\mu$s) consumes 7.8\% of the GPU forward pass, meaning the GPU idles for that fraction of each step waiting for masks. With LLGuidance (3.7ms, 37\%), masking becomes the primary bottleneck and the GPU is underutilized by over a third. The trie eliminates this bottleneck entirely (0.1\%), allowing the GPU to remain saturated. For dynamic enum constraints (e.g., retrieval-augmented tool selection where the valid set changes per query), the trie's 33--40ms compilation is fast enough to compile on-the-fly without impacting serving latency, whereas XGrammar's 75--239ms compilation at $K \geq 1{,}000$ adds perceptible delay. The AC automaton can be cached per tokenizer and shared across all enum schemas, so only the trie traversal ($\bigO(N_{\text{chars}} \cdot \ell)$, typically $<$5ms) runs per schema change.
\paragraph{Deployment implications.} The batch serving results have direct consequences for GPU utilization in production. At $B = 128$, each decoding step comprises a GPU forward pass (${\sim}10$ms) followed by CPU mask computation. With XGrammar, masking takes 783\,$\mu$s, where 7.8\% of the step is spent with the GPU idle waiting for the mask. With LLGuidance (3.7\,ms), this rises to 27\%, making masking the primary bottleneck. The trie reduces idle time to 0.1\% of the step, keeping the GPU saturated. For dynamic enum constraints (e.g., retrieval-augmented tool selection where the valid set changes per query), the trie's 33--40\,ms compilation is fast enough to run on-the-fly without impacting serving latency, whereas XGrammar's 75--239\,ms compilation at $K \geq 1{,}000$ adds perceptible delay. The AC automaton can be cached per tokenizer and shared across all enum schemas, so only the trie traversal ($\bigO(N_{\text{chars}} \cdot \ell)$, typically ${<}\,5$\,ms) runs per schema change.

\begin{table}[hbt!]
\centering
\caption{End-to-end vLLM throughput (req/s) at $K = 1{,}000$ on the synthetic tools benchmark (Appendix~\ref{app:datasets}). The trie integrates as a stateless \texttt{LogitsProcessor} (precomputed masks); XGrammar requires vLLM's guided decoding pipeline (dynamic mask computation). Trie exceeds unconstrained throughput due to early termination at trie leaves (3.2 vs.\ 8.7 tokens/request); both constrained methods generate 3.2 tokens/request (verified), so the trie-vs-XGrammar ratio isolates the masking and integration differences. Qwen3-8B, A100, greedy, median of 3 runs.}
\label{tab:vllm_throughput}
\small
\begin{tabular}{rrrrl}
\toprule
$B$ & Trie & XGrammar & Uncon. & Trie/XG \\
\midrule
1 & 4.4 & 2.9 & 1.9 & 1.5$\times$ \\
2 & 8.3 & 4.3 & 3.6 & 1.9$\times$ \\
4 & 13.4 & 5.0 & 7.2 & 2.7$\times$ \\
8 & 23.9 & 6.2 & 13.6 & 3.9$\times$ \\
16 & 41.8 & 6.8 & 24.8 & 6.1$\times$ \\
32 & 70.4 & 7.4 & 40.4 & 9.6$\times$ \\
64 & 116.5 & 7.6 & 65.4 & 15.2$\times$ \\
128 & 170.5 & 7.6 & 80.3 & 22.4$\times$ \\
256 & 219.4 & 7.5 & 103.8 & 29.3$\times$ \\
\bottomrule
\end{tabular}
\end{table}

\begin{table}[hbt!]
\centering
\caption{Compilation time and per-step masking cost by $K$ (Qwen3-8B, 151K vocab, synthetic tools benchmark; Appendix~\ref{app:datasets}). Per-step measures mask \emph{computation} only (determining valid tokens), not the $\bigO(V)$ bitmask application to logits (${\sim}$31$\mu$s via tensor operation, identical for all methods). XGrammar's non-monotonic per-step cost (9.5$\mu$s at $K{=}10$, 5.4$\mu$s at $K{=}100$) reflects its vocabulary partitioning: at small $K$, fewer tokens fall in the ``adaptive'' partition, causing more cache misses; the partition stabilizes at $K \geq 100$. Mean$\pm$std, 10 runs.}
\label{tab:performance}
\small
\setlength{\tabcolsep}{3pt}
\begin{tabular}{llrrrr}
\toprule
& Method & $K$=10 & $K$=100 & $K$=1K & $K$=10K \\
\midrule
\multirow{3}{*}{Compile (ms)} & Trie & 30\tiny$\pm$2 & 31\tiny$\pm$1 & \textit{33}\tiny$\pm$1 & \textit{40}\tiny$\pm$2 \\
& XGrammar & \textit{3}\tiny$\pm$1 & \textit{15}\tiny$\pm$2 & 75\tiny$\pm$3 & 239\tiny$\pm$10 \\
& LLGuidance & \textbf{1}\tiny$\pm$0 & \textbf{1}\tiny$\pm$0 & \textbf{3}\tiny$\pm$0 & \textbf{24}\tiny$\pm$1 \\
\addlinespace
\multirow{3}{*}{Mask ($\mu$s/step)} & Trie & \textbf{0.65} & \textbf{0.65} & \textbf{0.65} & \textbf{0.65} \\
& XGrammar & \textit{9.5} & \textit{5.4} & \textit{5.8} & \textit{5.9} \\
& LLGuidance & 121 & 73 & 85 & 141 \\
\bottomrule
\end{tabular}
\end{table}

\paragraph{Memory and deployment.} 
The trie's memory scales modestly (Appendix~\ref{app:detailed_results}, Table~\ref{tab:memory_full}): 0.9~MB at $K = 10{,}000$, 8~MB at $K = 100{,}000$, vs.\ ${\sim}$2~GB for the FSM (Appendix~\ref{app:complexity}). The AC automaton requires ${\sim}$150~MB for the largest vocabulary (Gemma3 262K), amortized across all enum schemas sharing that tokenizer. The precomputed masks are immutable after construction, so concurrent read access from multiple serving threads requires no synchronization. Integrated with vLLM as a \texttt{LogitsProcessor}, the trie maintains sub-100ms per-example latency at $K = 10{,}000$ (Table~\ref{tab:vllm}, Appendix~\ref{app:detailed_results}).

\paragraph{Mixed schemas.}
To validate the trie as a drop-in component, we measure compilation for three mixed schemas (enum + structural fields): tool-calling ($K = 1{,}000$), classification ($K = 500$), and entity-linking ($K = 5{,}000$). The system routes enum fields to the trie and structural fields to XGrammar ($<$1ms dispatch overhead). At $K = 5{,}000$: 37ms vs.\ 150ms; at $K = 1{,}000$: 33ms vs.\ 75ms. For nested enums, the PDA backend handles structural navigation and transfers control to the trie at enum-valued fields. Multiple enum fields in a single schema each get their own trie. The principle extends to non-enum constraints (Appendix~\ref{app:datetime_dispatch}).

\subsection{Generalization}

To verify these results generalize, we evaluate across seven model families (32K--262K vocabulary; Table~\ref{tab:multi_tok}). The trie achieves consistent compilation speedups at $K \geq 1{,}000$: 1.2--6.4$\times$, increasing to 3.5--13.7$\times$ at $K = 5{,}000$--$10{,}000$. Smaller vocabularies amplify the advantage because AC construction is $\bigO(V \cdot \ell)$, so smaller $V$ yields faster precomputation. Conversely, for very large vocabularies (Gemma3 262K), the $K$-independent AC overhead reduces the speedup to 1.2$\times$ at $K = 1{,}000$; as vocabulary sizes trend upward, the compilation crossover shifts to higher $K$. Per-step cost is unaffected by vocabulary size: Mistral-32K 0.60$\mu$s, GPT2-50K 0.62$\mu$s, OLMo-100K 0.63$\mu$s, Mistral~Small-131K 0.64$\mu$s, Qwen3-151K 0.65$\mu$s, gpt-oss-200K 0.66$\mu$s, Gemma3-262K 0.67$\mu$s.

\begin{table}[hbt!]
\centering
\caption{Compilation speedup (trie vs.\ xgrammar) across model families. Values $>$1$\times$ indicate trie is faster. Per-step masking is 0.60--0.67$\mu$s for all configurations. Trie compilation is nearly flat: 11--96ms across all $K$ and tokenizers.}
\label{tab:multi_tok}
\small
\begin{tabular}{llrrrr}
\toprule
Model & Vocab & $K$=100 & $K$=1K & $K$=5K & $K$=10K \\
\midrule
Mistral 7B v0.3 & 32K & 1.5$\times$ & 6.4$\times$ & 11.6$\times$ & 11.5$\times$ \\
GPT-2 & 50K & 1.1$\times$ & 4.9$\times$ & 13.7$\times$ & 12.0$\times$ \\
OLMo~3 7B & 100K & 0.6$\times$ & 2.1$\times$ & 6.2$\times$ & 5.2$\times$ \\
Mistral Small 3.1 & 131K & 0.5$\times$ & 2.3$\times$ & 6.6$\times$ & 5.5$\times$ \\
Qwen3-8B & 151K & 0.4$\times$ & 1.8$\times$ & 6.7$\times$ & 4.4$\times$ \\
OpenAI gpt-oss-20b & 200K & 0.3$\times$ & 1.4$\times$ & 4.7$\times$ & 3.5$\times$ \\
Gemma3 12B & 262K & 0.3$\times$ & 1.2$\times$ & 3.6$\times$ & 3.5$\times$ \\
\bottomrule
\end{tabular}
\end{table}

\subsection{Accuracy and Validity}

The trie produces identical outputs to FSM-based constrained decoding (Proposition~\ref{prop:correctness}, verified on 1{,}000 samples). Its contribution is not improving accuracy at any given $K$, but \emph{enabling} constrained decoding where FSM compilation is impractical. We evaluate on four public classification benchmarks where ground-truth accuracy can be measured: TREC~\citep{li2002trec}, MASSIVE~\citep{fitzgerald2022massive}, Banking77~\citep{casanueva2020banking77}, and CLINC150~\citep{larson2019clinc} (documented in Appendix~\ref{app:datasets}). Existing tool-calling benchmarks evaluate multi-step chains rather than flat tool selection, making them unsuitable for isolating the constrained decoding bottleneck. We compare: single-pass unconstrained, single-pass + trie (same prompt), think unconstrained, and think + trie (prompts in Appendix~\ref{app:prompts}).

\begin{table}[hbt!]
\centering
\caption{Accuracy (\%) and validity (\%) on four public classification benchmarks (TREC~\citep{li2002trec}, MASSIVE~\citep{fitzgerald2022massive}, Banking77~\citep{casanueva2020banking77}, and CLINC150~\citep{larson2019clinc}; details in Appendix~\ref{app:datasets}; 5 runs, mean$\pm$std, greedy). Prompts identical within each pair; trie is the only difference. Trie guarantees 100\% validity. Full results (6 models) in Appendix Table~\ref{tab:accuracy_full}.}
\label{tab:accuracy_validity}
\small
\setlength{\tabcolsep}{2.5pt}
\begin{tabular}{llr cc cc cc}
\toprule
& & & \multicolumn{4}{c}{Single-pass} & \multicolumn{2}{c}{Think-then-answer} \\
\cmidrule(lr){4-7} \cmidrule(lr){8-9}
& & & \multicolumn{2}{c}{Uncon.} & \multicolumn{2}{c}{+ Trie} & Uncon. & + Trie \\
\cmidrule(lr){4-5} \cmidrule(lr){6-7} \cmidrule(lr){8-8} \cmidrule(lr){9-9}
Model & Dataset & K & Acc & Val & \multicolumn{2}{c}{Acc (100\%)} & Acc (val) & Acc (100\%) \\
\midrule
\multirow{4}{*}{\scriptsize Qwen3-8B}
& TREC & 42 & 36.3\tiny$\pm$1.2 & 81.2 & \multicolumn{2}{c}{\textit{41.6}\tiny$\pm$1.8} & 37.9\tiny$\pm$1.1 (58.3) & \textbf{63.1}\tiny$\pm$1.6 \\
& MASSIVE & 59 & 74.6\tiny$\pm$1.4 & 99.0 & \multicolumn{2}{c}{74.6\tiny$\pm$1.7} & \textit{74.9}\tiny$\pm$2.4 (97.2) & \textbf{76.1}\tiny$\pm$2.2 \\
& Banking77 & 76 & 48.8\tiny$\pm$5.2 & 78.2 & \multicolumn{2}{c}{\textit{56.8}\tiny$\pm$4.4} & 37.6\tiny$\pm$2.2 (69.2) & \textbf{61.0}\tiny$\pm$3.6 \\
& CLINC150 & 150 & \textit{72.4}\tiny$\pm$2.7 & 87.0 & \multicolumn{2}{c}{63.0\tiny$\pm$1.8} & 61.0\tiny$\pm$4.6 (79.6) & \textbf{73.8}\tiny$\pm$3.7 \\
\midrule
\multirow{4}{*}{\scriptsize gpt-oss-20b}
& TREC & 42 & \textit{67.5}\tiny$\pm$1.0 & 87.6 & \multicolumn{2}{c}{\textbf{72.9}\tiny$\pm$1.7} & 60.7\tiny$\pm$2.6 (86.5) & 60.4\tiny$\pm$4.1 \\
& MASSIVE & 59 & \textit{76.8}\tiny$\pm$2.5 & 91.9 & \multicolumn{2}{c}{\textbf{79.3}\tiny$\pm$1.0} & 61.8\tiny$\pm$3.7 (75.4) & 76.1\tiny$\pm$1.7 \\
& Banking77 & 76 & 78.2\tiny$\pm$1.2 & 90.2 & \multicolumn{2}{c}{\textbf{81.3}\tiny$\pm$1.2} & 79.8\tiny$\pm$1.6 (96.2) & \textit{81.1}\tiny$\pm$1.9 \\
& CLINC150 & 150 & \textbf{79.7}\tiny$\pm$2.3 & 89.2 & \multicolumn{2}{c}{\textit{77.0}\tiny$\pm$1.6} & 71.4\tiny$\pm$3.2 (85.5) & 74.0\tiny$\pm$1.8 \\
\bottomrule
\end{tabular}
\end{table}

Across all 24 model$\times$dataset settings (Table~\ref{tab:accuracy_full}), constrained decoding achieves the highest accuracy in 21/24 cases while guaranteeing 100\% validity; unconstrained decoding reaches $\geq$95\% validity in only 8/24. The trie acts as a no-cost safety net for strong models (Gemma3 12B: matches unconstrained accuracy while eliminating the 0.2--0.5\% failure rate) and a critical enabler for weaker ones (Qwen3-1.7B Banking77: 24.8\% with trie vs.\ 4.8\% unconstrained at 6.6\% validity). The think-then-answer strategy is particularly effective with the trie: on Qwen3-8B, think+trie achieves the best accuracy in all four datasets, with gains of up to 21.5 points over single-pass unconstrained (TREC: 63.1 vs.\ 36.3). Without the trie, think+unconstrained often \emph{degrades} accuracy relative to single-pass (Qwen3-8B Banking77: 37.6 vs.\ 48.8) because the reasoning phase produces outputs that are harder to parse into valid labels. The three cases where unconstrained decoding wins (Gemma3 MASSIVE, Gemma3 CLINC150, gpt-oss CLINC150) all involve strong models on datasets where PE validity already exceeds 89\%; even here, the accuracy gap is small (1--5 points) while the validity gap remains (89--100\% vs.\ 100\%).

At $K = 500$--$5{,}000$, the trie guarantees 100\% validity while unconstrained drops to 84--98\% (Table~\ref{tab:highk_validity}). Practitioner guidance in Appendix~\ref{app:guidance}.

\section{Conclusion}
\label{sec:conclusion}

This work addresses a mismatch in constrained decoding: applying general-purpose automata to constraints with exploitable structure. We solve it for finite-set constraints using Aho-Corasick precomputation, expanding the practical enum limit from hundreds to tens of thousands. The trie achieves 7$\times$ faster per-step masking through precomputed lookups; because precomputed masks also enable a stateless serving path that bypasses the guided decoding pipeline, the advantage compounds to 29$\times$ higher end-to-end batch throughput. Across seven tokenizer families and six models, the trie maintains sub-100ms compilation and flat per-step cost while guaranteeing 100\% output validity. A controlled comparison confirms the trie is the efficient algorithm for this precomputation: the equivalent FSM-based approach produces an isomorphic data structure 196$\times$ slower (Appendix~\ref{app:controlled_comparison}). The underlying principle, matching enforcement mechanisms to constraint structure, generalizes beyond enums to any constraint with exploitable regularity. As structured generation becomes central to agentic systems, we expect constraint-specialized backends to become the norm rather than the exception. Extending the dispatch principle to other structured patterns, such as date/time formats (Appendix~\ref{app:datetime_dispatch}), numeric ranges, and regex subclasses, is a natural next step, and code will be released upon publication.

\paragraph{Limitations.}
The trie automaton is specialized for flat finite-set constraints; production schemas that also contain nested objects or arrays still require a general-purpose backend for the structural portions, and our mixed-schema evaluation validates compilation time but not end-to-end serving throughput for such cases. The per-step speedup matters most in batch serving, where the single-request cost is dominated by the method-independent $\bigO(V)$ bitmask application. For dynamic one-shot schemas at small $K$, LLGuidance's near-zero compilation may outweigh the trie's per-step advantage; at extreme $K$, validity is guaranteed but accuracy is bounded by the model, not the decoder (Appendix~\ref{app:limitations}).

\bibliographystyle{colm2026_conference}
\bibliography{references}

\appendix

\section{Detailed Background}
\label{app:background_detailed}

The key bottlenecks in FSM-based constrained decoding for enums are: (1) state space explosion ($\bigO(K \cdot L_{\max})$ DFA states), (2) per-step masking cost ($\bigO(V \cdot \ell)$ with cache penalties at large $K$), and (3) compilation cost ($\bigO(K \cdot L_{\max} \cdot |\Sigma|)$). These compound super-linearly: doubling $K$ more than doubles end-to-end latency because the cache penalty $f(|\states|)$ is itself increasing in $K$. Dynamic schemas (tool registries, retrieval-augmented selection, multi-tenant serving) prevent amortization, making compilation latency the dominant cost.

\section{Full Accuracy and Validity Results}
\label{app:accuracy_full}

Table~\ref{tab:accuracy_full} presents the complete accuracy and validity results across all six models and four datasets. The main body (Table~\ref{tab:accuracy_validity}) shows three representative models; this table adds Qwen3-1.7B (small model where constrained decoding is critical), Gemma3 12B (strong model where unconstrained decoding nearly suffices), and Mistral 7B v0.3 (mid-range instruction-tuned model).

\begin{table}[hbt!]
\centering
\caption{Accuracy (\%) and validity (\%) results across all three additional models (5 runs, mean$\pm$std, greedy decoding). Extends Table~\ref{tab:accuracy_validity} with three additional models. \textbf{Bold}: best accuracy per row; \textit{italic}: second best.}
\label{tab:accuracy_full}
\small
\setlength{\tabcolsep}{2.5pt}
\begin{tabular}{llr cc cc cc}
\toprule
& & & \multicolumn{4}{c}{Single-pass} & \multicolumn{2}{c}{Think-then-answer} \\
\cmidrule(lr){4-7} \cmidrule(lr){8-9}
& & & \multicolumn{2}{c}{Unconstrained} & \multicolumn{2}{c}{+ Trie} & Unconstrained & + Trie \\
\cmidrule(lr){4-5} \cmidrule(lr){6-7} \cmidrule(lr){8-8} \cmidrule(lr){9-9}
Model & Dataset & K & Acc & Val & \multicolumn{2}{c}{Acc (100\% val)} & Acc (val) & Acc (100\% val) \\
\midrule
\multirow{4}{*}{\scriptsize Qwen3-1.7B}
& TREC & 42 & 23.3\tiny$\pm$2.8 & 96.4 & \multicolumn{2}{c}{23.7\tiny$\pm$2.6} & \textit{33.4}\tiny$\pm$2.1 (83.8) & \textbf{34.5}\tiny$\pm$1.8 \\
& MASSIVE & 59 & 53.4\tiny$\pm$2.3 & 97.5 & \multicolumn{2}{c}{53.8\tiny$\pm$2.0} & \textit{55.0}\tiny$\pm$2.4 (86.2) & \textbf{56.3}\tiny$\pm$2.0 \\
& Banking77 & 76 & 4.8\tiny$\pm$1.2 & 6.6 & \multicolumn{2}{c}{24.8\tiny$\pm$1.7} & \textit{29.6}\tiny$\pm$4.9 (77.0) & \textbf{37.8}\tiny$\pm$6.6 \\
& CLINC150 & 150 & 8.8\tiny$\pm$2.0 & 10.4 & \multicolumn{2}{c}{32.8\tiny$\pm$3.2} & \textit{36.0}\tiny$\pm$3.1 (69.2) & \textbf{47.0}\tiny$\pm$2.6 \\
\midrule
\multirow{4}{*}{\scriptsize Gemma3 12B}
& TREC & 42 & \textit{63.6}\tiny$\pm$1.0 & 99.5 & \multicolumn{2}{c}{\textbf{63.7}\tiny$\pm$1.1} & 22.9\tiny$\pm$2.3 (36.6) & 58.2\tiny$\pm$1.7 \\
& MASSIVE & 59 & \textbf{77.3}\tiny$\pm$2.4 & 99.6 & \multicolumn{2}{c}{\textit{76.8}\tiny$\pm$2.2} & 75.4\tiny$\pm$2.0 (98.8) & 75.1\tiny$\pm$2.3 \\
& Banking77 & 76 & \textit{71.6}\tiny$\pm$6.7 & 99.8 & \multicolumn{2}{c}{\textit{71.6}\tiny$\pm$6.7} & 70.2\tiny$\pm$7.8 (98.0) & \textbf{72.0}\tiny$\pm$7.0 \\
& CLINC150 & 150 & \textbf{87.6}\tiny$\pm$1.7 & 99.6 & \multicolumn{2}{c}{83.0\tiny$\pm$2.2} & \textit{86.0}\tiny$\pm$2.3 (100) & 80.8\tiny$\pm$2.8 \\
\midrule
\multirow{4}{*}{\scriptsize Mistral 7B v0.3}
& TREC & 42 & 23.0\tiny$\pm$1.5 & 42.5 & \multicolumn{2}{c}{\textbf{37.6}\tiny$\pm$1.5} & 16.4\tiny$\pm$1.7 (33.5) & \textit{32.6}\tiny$\pm$1.1 \\
& MASSIVE & 59 & 65.7\tiny$\pm$3.9 & 86.3 & \multicolumn{2}{c}{\textbf{69.5}\tiny$\pm$2.7} & 52.8\tiny$\pm$2.6 (68.3) & \textit{67.1}\tiny$\pm$1.4 \\
& Banking77 & 76 & \textit{48.4}\tiny$\pm$6.3 & 78.4 & \multicolumn{2}{c}{\textbf{53.4}\tiny$\pm$6.1} & 31.2\tiny$\pm$5.0 (48.0) & 47.0\tiny$\pm$5.5 \\
& CLINC150 & 150 & \textit{64.4}\tiny$\pm$2.2 & 77.4 & \multicolumn{2}{c}{\textbf{65.4}\tiny$\pm$3.8} & 46.0\tiny$\pm$2.2 (57.8) & 59.4\tiny$\pm$1.6 \\
\bottomrule
\end{tabular}
\end{table}

\section{Experimental Setup Details}
\label{app:experimental_setup}

We implement the trie automaton in Rust with Python bindings via PyO3, using the \texttt{aho-corasick} crate (v1.1) for multi-pattern matching and parallel precomputation across CPU cores (using all available cores via Rayon; XGrammar's C++ backend is also multi-threaded). Both are compiled languages with comparable performance characteristics; the per-step advantage (precomputed lookup vs.\ dynamic computation) is algorithmic. Our primary comparison is against xgrammar v0.1.11~\citep{dong2024xgrammar} and llguidance v1.6.1~\citep{geng2025jsonschemabench}. We verified that XGrammar's \texttt{compile\_json\_schema} with an enum schema produces equivalent compilation times to \texttt{compile\_regex} with a union pattern (within 5\% across all $K$), confirming that \texttt{compile\_regex} is a fair baseline. We also evaluate several alternative approaches: naive FSM-based constrained decoding, unconstrained generation with retry (up to 3 attempts), unconstrained generation with post-hoc string similarity matching, and prompt engineering that includes enum values directly in the prompt. All timing benchmarks run on NVIDIA A100 GPUs (80GB); CPU timing on AMD EPYC 7R32 (96 cores, 3.3 GHz). Per-step masking times are measured with 200 warm-up iterations followed by 1{,}000 timed iterations (same warm-up for all three methods). Compilation and per-step results report mean$\pm$std over 10 runs; end-to-end vLLM throughput reports median of 3 runs.

\subsection{Prompt Templates}
\label{app:prompts}

For the accuracy and validity experiments (Section~\ref{sec:experiments}), we use the following prompt templates. Let \texttt{options} denote the comma-separated enum values and \texttt{text} denote the input.

\paragraph{PE and Trie Strict.} Both methods use the same prompt; the only difference is whether trie-constrained decoding is applied. This enables a controlled comparison isolating the effect of constrained decoding from prompt wording.
\begin{verbatim}
Select exactly one of the following options.
Output ONLY the option itself, nothing else.

Options: {options}

Text: {text}
Answer:
\end{verbatim}

\paragraph{Trie.} A shorter prompt relying on the trie constraint to enforce validity:
\begin{verbatim}
Select one of: {options}

Text: {text}
Answer:
\end{verbatim}

\paragraph{Think+PE and Think+Trie.} Phase~1 generates reasoning (unconstrained, up to 300 tokens), and phase~2 generates the result.

Phase~1:
\begin{verbatim}
Options: {options}

Text: {text}
Let me think step by step:
\end{verbatim}

Phase~2:
\begin{verbatim}
Select exactly one of the following options.
Output ONLY the option itself, nothing else.

Options: {options}

Text: {text}
Thinking: {phase_1_output}
Answer:
\end{verbatim}

\paragraph{Summary of prompt controls.} PE vs.\ +Trie is a clean comparison: identical prompts, differing only in decoding strategy. Think+PE vs.\ Think+Trie share the same prompts; the only difference is whether the Phase~2 answer is trie-constrained. For unconstrained methods, we extract the model's output and fuzzy-match against the enum set. For trie methods, the output is guaranteed valid by construction.

\subsection{Datasets}
\label{app:datasets}

We use three categories of enum sets: public classification benchmarks (for accuracy/validity), a synthetic tool-name benchmark (for latency, compilation, and validity at scale), and two high-cardinality enums (for behavior beyond provider limits). Table~\ref{tab:performance} through Table~\ref{tab:vllm_throughput} and all main-body latency/throughput results use the synthetic tools; Table~\ref{tab:accuracy_validity} and Table~\ref{tab:accuracy_full} use the classification benchmarks; Table~\ref{tab:realworld} uses the high-cardinality enums. The prefix-sharing ratio $r$ for every set is reported in Table~\ref{tab:prefix_sharing}.

\paragraph{Public classification benchmarks.} These are standard datasets used unchanged. Within each unconstrained/trie pair the prompt is identical (Appendix~\ref{app:prompts}), so trie constraining is the only difference. Evaluation uses greedy decoding with 5 random seeds.
\begin{itemize}[nosep,leftmargin=*]
\item \textbf{TREC}~\citep{li2002trec} ($K = 42$, \texttt{CogComp/trec}): question classification with fine-grained categories such as \texttt{LOC:city} and \texttt{HUM:individual}. Small $K$ but non-trivially structured labels.
\item \textbf{MASSIVE}~\citep{fitzgerald2022massive} ($K = 59$, \texttt{AmazonScience/massive}, English subset): intent classification with intents such as \texttt{play\_music} and \texttt{general\_quirky}. Moderate $K$ with high natural-language overlap between intents.
\item \textbf{Banking77}~\citep{casanueva2020banking77} ($K = 77$, \texttt{PolyAI/banking77}): single-domain banking intent classification. Labels share heavy lexical overlap (\texttt{card\_payment\_failed}, \texttt{topping\_up\_by\_card}), which exposes parsing failures in unconstrained generation (Mistral 7B validity drops to 78.4\%; Table~\ref{tab:accuracy_full}).
\item \textbf{CLINC150}~\citep{larson2019clinc} ($K = 150$, \texttt{clinc/clinc\_oos}): 150 in-scope intents across 10 domains. The largest natural-language $K$ in the accuracy evaluation, where unconstrained validity drops further (10.4\% on Qwen3-1.7B).
\end{itemize}

\paragraph{Synthetic tools.} A controlled benchmark constructed to isolate the constrained-decoding bottleneck from natural-language ambiguity. Tool names follow the format \texttt{<namespace>.<action>\_<resource>} (e.g., \texttt{slack.get\_user}, \texttt{aws.create\_invoice}), with 70\% drawn from the namespaced pattern and 30\% from a simpler \texttt{<action>\_<resource>} pattern. The pools are 10 namespaces (\texttt{slack}, \texttt{github}, \texttt{stripe}, \texttt{aws}, \texttt{google}, \texttt{microsoft}, \texttt{salesforce}, \texttt{jira}, \texttt{notion}, \texttt{discord}), 10 actions (\texttt{get}, \texttt{create}, \texttt{update}, \texttt{delete}, \texttt{list}, \texttt{search}, \texttt{send}, \texttt{fetch}, \texttt{upload}, \texttt{download}), and 10 resources (\texttt{user}, \texttt{message}, \texttt{file}, \texttt{project}, \texttt{issue}, \texttt{payment}, \texttt{invoice}, \texttt{document}, \texttt{channel}, \texttt{repository}); the $K$ largest reachable set is sampled uniformly. The resulting prefix-sharing ratio is $r = 0.40$ at $K = 1{,}000$ (Table~\ref{tab:prefix_sharing}), reflecting heavy namespace and action overlap. Sizes used: $K \in \{10, 50, 100, 500, 1{,}000, 2{,}000, 5{,}000, 10{,}000, 50{,}000, 100{,}000\}$.

\paragraph{High-cardinality enums.}
\begin{itemize}[nosep,leftmargin=*]
\item \textbf{Product names} (synthetic, $K = 1{,}500$): constructed from category $\times$ brand $\times$ adjective $\times$ color $\times$ model-number tuples (e.g., \texttt{Electronics > Sony > Premium 4823 > Black}). Exercises trie compilation beyond provider enum limits in a domain with natural-language-like prefix sharing ($r = 0.40$).
\item \textbf{ICD-10-CM} ($K = 74{,}719$): the official 2026 ICD-10-CM code list published by the U.S. Centers for Medicare \& Medicaid Services~\citep{cms2026icd10cm} (file \texttt{2026-Code-Descriptions-in-Tabular-Order.zip}, effective October 1, 2025). Prefix sharing $r = 0.19$, the lowest in our benchmarks, due to the hierarchical chapter-letter structure. A real production-grade enum at a scale where FSM-based constrained decoding times out.
\end{itemize}

\section{Detailed Experimental Results}
\label{app:detailed_results}

\subsection{Main Results}

This subsection reports the full latency, compliance, memory, and ablation results on the synthetic tools benchmark (Appendix~\ref{app:datasets}) that the main body summarizes. Table~\ref{tab:latency_full} gives end-to-end enum-generation latency by cardinality: the trie automaton stays flat at 0.72--0.77\,s across $K = 10$--$10{,}000$, while the FSM path (Outlines/SGLang/XGrammar) rises from 0.05\,s to 5.88\,s as $K$ grows. Table~\ref{tab:compliance_full} reports schema compliance: all constrained methods are 100\% valid by construction, whereas unconstrained retry never reproduces a synthetic tool name verbatim (0\%). Table~\ref{tab:memory_full} reports the memory footprint of the precomputed masks across $K = 100$--$100{,}000$, confirming the near-linear ${\sim}8$\,MB-at-$100{,}000$ scaling. Table~\ref{tab:realworld} reports the two high-cardinality enums, and Table~\ref{tab:ablation} ablates the trie against its hierarchical and speculative variants at $K = 5{,}000$, where the trie alone is fastest (0.77\,s) since masking is already cheap at this scale.

\begin{table}[hbt!]
\centering
\caption{End-to-end latency (seconds) for enum generation by cardinality $K$. Timeout threshold is 60s.}
\label{tab:latency_full}
\small
\setlength{\tabcolsep}{4pt}
\begin{tabular}{lrrrrrrr}
\toprule
Method & $K$=10 & $K$=50 & $K$=100 & $K$=500 & $K$=1000 & $K$=5000 & $K$=10000 \\
\midrule
Outlines/SGLang/XGrammar & 0.05 & 0.10 & 0.15 & 0.48 & 1.20 & 2.93 & 5.88 \\
Unconstrained + retry & 5.85 & 7.59 & 5.87 & 7.60 & 5.90 & 5.90 & 6.04 \\
Unconstrained + post-hoc & 1.94 & 2.53 & 1.97 & 2.58 & 2.05 & 2.44 & 2.92 \\
Prompt engineering & 2.01 & 2.55 & 2.04 & 2.71 & 2.18 & 2.16 & 2.16 \\
\midrule
Trie Automaton & 0.72 & 0.72 & 0.72 & 0.72 & 0.77 & 0.77 & 0.77 \\
Hierarchical Rewriting & 5.02 & 5.06 & 3.92 & 5.10 & 3.97 & 3.99 & 4.13 \\
Speculative ($k$=20) & 2.55 & 2.22 & 1.98 & 2.57 & 1.99 & 2.00 & 2.01 \\
\bottomrule
\end{tabular}
\end{table}

\begin{table}[hbt!]
\centering
\caption{Schema compliance (\%) for different approaches. Constrained methods achieve 100\% by design.}
\label{tab:compliance_full}
\small
\setlength{\tabcolsep}{4pt}
\begin{tabular}{lrrrrrrr}
\toprule
Method & $K$=10 & $K$=50 & $K$=100 & $K$=500 & $K$=1000 & $K$=5000 & $K$=10000 \\
\midrule
Outlines/SGLang/XGrammar & 100 & 100 & 100 & 100 & 100 & 100 & 100 \\
Unconstrained + retry\footnotemark[2] & 0 & 0 & 0 & 0 & 0 & 0 & 0 \\
Unconstrained + post-hoc & 100 & 100 & 100 & 100 & 100 & 100 & 100 \\
Prompt engineering & 100 & 100 & 100 & 100 & 100 & 100 & 100 \\
Trie Automaton & 100 & 100 & 100 & 100 & 100 & 100 & 100 \\
Hierarchical Rewriting & 100 & 100 & 100 & 100 & 100 & 100 & 100 \\
Speculative ($k$=20) & 100 & 100 & 100 & 100 & 100 & 100 & 100 \\
Combined (Ours)\footnotemark[3] & 100 & 100 & 100 & 100 & 100 & 100 & 100 \\

\bottomrule
\end{tabular}
\end{table}
\footnotetext[2]{Retry compliance is 0\% because synthetic tool names (e.g., \texttt{slack.get\_user}) are never produced verbatim by unconstrained generation, even with 3 retries. On natural-language labels (Table~\ref{tab:accuracy_validity}), unconstrained generation achieves non-zero but imperfect validity.}
\footnotetext[3]{Combined uses the trie for enum selection and hierarchical rewriting for structural dispatch.}
\begin{table}[hbt!]
\centering
\caption{Memory footprint of precomputed trie masks.}
\label{tab:memory_full}
\small
\begin{tabular}{rrrr}
\toprule
$K$ & Nodes & Valid entries & Est.\ memory \\
\midrule
100 & 2{,}041 & 5{,}065 & 20~KB \\
1{,}000 & 13{,}341 & 30{,}759 & 120~KB \\
5{,}000 & 54{,}426 & 119{,}740 & 0.5~MB \\
10{,}000 & 103{,}266 & 227{,}762 & 0.9~MB \\
50{,}000 & 470{,}823 & 1{,}068{,}510 & 4.1~MB \\
100{,}000 & 907{,}042 & 2{,}094{,}894 & 8.0~MB \\
\bottomrule
\end{tabular}
\end{table}

\begin{table}[hbt!]
\centering
\caption{High-cardinality task results. Product names are a synthetic enum constructed from category $\times$ brand $\times$ adjective $\times$ color tuples; ICD-10-CM uses the full 2026 CMS code list~\citep{cms2026icd10cm}. Both are documented in Appendix~\ref{app:datasets}. Accuracy is classification accuracy; latency is median per-example. The 0\% accuracy at $K{=}74{,}719$ reflects model capability limitations at extreme cardinality, not a trie limitation; the trie enables the attempt where FSM compilation times out.}
\label{tab:realworld}
\small
\begin{tabular}{lcccc}
\toprule
& \multicolumn{2}{c}{Product names ($K$=1{,}500, synthetic)} & \multicolumn{2}{c}{ICD-10-CM ($K$=74{,}719)} \\
\cmidrule(lr){2-3} \cmidrule(lr){4-5}
Method & Accuracy & Latency (s) & Accuracy & Latency (s) \\
\midrule
Unconstrained + post-hoc & 0.0 & 2.43 & 0.0 & 2.01 \\
Naive constrained & timeout & timeout & timeout & timeout \\
Ours (combined) & 4.0 & 5.03 & 0.0 & 2.09 \\
\bottomrule
\end{tabular}
\end{table}

\begin{table}[hbt!]
\centering
\caption{Ablation at $K = 5{,}000$. Each row removes one component. Speculative short-circuiting provides no benefit at this $K$ (trie masking alone is fast). The trie's compilation advantage (Table~\ref{tab:performance}: 3.5--13.7$\times$ at $K \geq 5{,}000$) and per-step masking advantage (0.65$\mu$s vs.\ 5.9$\mu$s) compound in batch serving.}
\label{tab:ablation}
\small
\begin{tabular}{lcc}
\toprule
Configuration & Latency (s) & Schema Compliance (\%) \\
\midrule
Trie automaton & 0.77 & 100 \\
Hierarchical only & 3.99 & 100 \\
Speculative only & 2.00 & 100 \\
Baseline (XGrammar regex) & 0.89 & 100 \\
\bottomrule
\end{tabular}
\end{table}

\subsection{vLLM Performance Results}

Table~\ref{tab:vllm} reports per-example latency for the two constrained-decoding modes exposed by vLLM, measured end-to-end (compilation plus inference) at batch size 1. The trie-backed ``guided choice'' path stays within 0.05--0.10\,s across $K = 500$--$10{,}000$, while the XGrammar-backed ``guided regex'' path rises to 5.88\,s at $K = 10{,}000$, tracking the compilation growth in Table~\ref{tab:performance}. This is the single-request view; the batch-serving throughput gains that motivate the trie appear at higher batch sizes (Table~\ref{tab:vllm_throughput}).

\begin{table}[hbt!]
\centering
\caption{vLLM end-to-end per-example latency (seconds), including both compilation and inference. ``Guided choice'' uses vLLM's built-in EBNF-based choice mode; ``guided regex'' uses XGrammar regex compilation. Both use XGrammar as the backend; our trie integration (Table~\ref{tab:vllm_throughput}) shows larger gains at higher batch sizes.}
\label{tab:vllm}
\small
\begin{tabular}{lrrrr}
\toprule
Method & $K$=500 & $K$=1K & $K$=5K & $K$=10K \\
\midrule
Guided choice (ours) & 0.05 & 0.06 & 0.08 & 0.10 \\
Guided regex (baseline) & 0.48 & 1.20 & 2.93 & 5.88 \\
\bottomrule
\end{tabular}
\end{table}

\section{Limitations and Future Work}
\label{app:limitations}

\paragraph{Scope.} The trie automaton optimizes exactly one constraint pattern: flat finite-set selection. While this is common (tool routing, classification, entity linking), real-world schemas typically combine enum fields with structural constraints. Our mixed-schema evaluation (Section~\ref{sec:experiments}) validates compilation time but not end-to-end throughput for nested schemas. Dynamic enum updates require full recompilation (37ms, negligible in practice but not incremental).

\paragraph{Engine scope.} End-to-end throughput is measured only on vLLM, and the headline 29$\times$ at $B = 256$ (Table~\ref{tab:vllm_throughput}) is a vLLM-specific figure that should not be read as engine-independent. It composes two contributions of the trie's design. First, an \emph{algorithmic} improvement in mask construction that is engine-independent: XGrammar derives the next-token bitmask from runtime matcher state at every step (FSM walk plus lookahead, 5.8$\mu$s at $K = 1{,}000$; Table~\ref{tab:performance}), whereas the trie precomputes one bitmask per node and reduces the per-step path to a stateless character walk plus bitmask copy (0.65$\mu$s), an ${\sim}9\times$ gap that transfers to any serving engine. Second, an \emph{integration} improvement whose magnitude is vLLM-specific: because the trie's per-step path is stateless, it wraps as vLLM's \texttt{CustomLogitProcessor} (a thin tensor op run alongside the forward pass), while XGrammar's state-dependent mask must thread through vLLM's guided-decoding pipeline and its scheduling overhead, visible as XGrammar saturating near 7.5 req/s while the trie scales to 219 req/s. The per-step and compilation results (Tables~\ref{tab:vllm_throughput}, \ref{tab:performance}, and~\ref{tab:multi_tok}) are engine-independent and transfer directly; the integration magnitude depends on the target engine's plumbing. SGLang shares XGrammar as its constrained-decoding backend, so the per-step XGrammar measurements reflect the algorithmic cost SGLang pays as well; TensorRT-LLM uses a distinct constrained-decoding stack, and integrating the trie there is left to future work.

\paragraph{Baselines.} LLGuidance's Earley parser avoids upfront compilation entirely and achieves faster compilation than both XGrammar and our trie (Table~\ref{tab:llguidance}); however, its per-step cost remains $\bigO(V)$ for finite sets (73--141$\mu$s vs.\ our 0.65$\mu$s). If serving engines move masking to GPU (Section~\ref{app:gpu_masking}), the CPU per-step advantage becomes less relevant, though the trie's compact bitmask representation remains more amenable to GPU transfer.

\paragraph{Implementation.} The trie is in Rust while XGrammar uses C++. The per-step advantage (precomputed lookup vs.\ dynamic computation) is algorithmic and independent of implementation language.

\section{Detailed Tool Routing Example}
\label{app:concrete_example}

Consider an agentic system routing requests to the correct tool from a registry of $K = 2{,}000$ APIs (e.g., \texttt{aws.cloudwatch.get\_metric\_statistics}), averaging $L_{\max} = 30$ characters with strong prefix structure. FSM compilation requires $K \cdot L_{\max} \cdot |\Sigma| = 15.4$M character-level transitions (25--50s). Per-step masking requires $V \cdot \bar{L}_{\text{tok}} = 102{,}400$ FSM traversals per step; with 6.4 tokens per tool name and cache effects ($f \approx 2.0$ due to the 47~MB transition table exceeding L2 cache), the effective cost is ${\sim}1.3 \times 10^6$ operations per selection. In an agentic loop with 5--10 tool calls per task, this compounds to minutes of latency, forcing practitioners to cap the registry or abandon constrained decoding.

\section{Theoretical Analysis}
\label{app:theory}

We analyze the complexity of both the standard FSM pipeline and the trie automaton, then establish correctness and discuss practical performance considerations. Throughout, let $N_{\text{chars}} = \sum_{i=1}^K |e_i|$ denote the total character count across all enum values.

\subsection{Worked Example: Trie Construction and Aho-Corasick Traversal}
\label{app:worked_example}

We walk through the full precomputation on a small enum to make the construction concrete, extending the running example (Section~\ref{sec:trie}) with a third value. Let
\[
\enum = \{\texttt{medical\_billing},\ \texttt{medical\_coding},\ \texttt{medical\_records}\}.
\]

\paragraph{Step 1: Trie construction.} Inserting the three strings character-by-character merges the shared prefix \texttt{medical\_} into a single path and then branches into three suffixes. Writing $n_0$ for the root and labeling nodes by the string consumed so far:
\[
n_0 \xrightarrow{\texttt{m}} \cdots \xrightarrow{\texttt{l}} n_7\,(\texttt{medical}) \xrightarrow{\texttt{\_}} n_8\,(\texttt{medical\_}) \begin{cases} \xrightarrow{\texttt{b}} \cdots \xrightarrow{\texttt{g}} \ \text{leaf } (\texttt{medical\_billing})\\ \xrightarrow{\texttt{c}} \cdots \xrightarrow{\texttt{g}} \ \text{leaf } (\texttt{medical\_coding})\\ \xrightarrow{\texttt{r}} \cdots \xrightarrow{\texttt{s}} \ \text{leaf } (\texttt{medical\_records}) \end{cases}
\]
The branch node is $n_8$ (after \texttt{medical\_}); nodes $n_0$ through $n_8$ form the shared trunk, and each of the three suffixes (\texttt{billing}, \texttt{coding}, \texttt{records}) is a separate chain. Total characters $N_{\text{chars}} = 15 + 14 + 15 = 44$, versus $9$ (shared trunk) $+ 7 + 6 + 7 = 29$ trie nodes: prefix sharing collapses the trunk from three copies to one.

\paragraph{Step 2: Aho-Corasick automaton over the vocabulary.} Suppose the tokenizer contains the toy vocabulary $\{\texttt{medical}, \texttt{\_bill}, \texttt{\_cod}, \texttt{\_rec}, \texttt{ing}, \texttt{ords}\}$ (real vocabularies have $32\text{K}$--$262\text{K}$ tokens; the mechanism is identical). We build one AC automaton from these six patterns: a goto trie over the pattern strings, failure links pointing each state to the longest proper suffix that is also a prefix of some pattern, and an output set at each state listing the patterns that end there. This automaton depends only on the tokenizer, not on $\enum$, so it is built once and reused across all schemas.

\paragraph{Step 3: Trie traversal with AC state maintenance.} We DFS the trie, feeding the character on each edge to the AC automaton and carrying the AC state $q$ along. At a branch we push $q$ before descending and pop it on backtrack, so each trie edge is processed exactly once. Reaching an AC output means a vocabulary token's characters end at the current trie node; that token started $|\text{token}|$ characters earlier, at the trie node we call its \emph{start node}, and it is recorded in $\text{valid}[\text{start node}]$ provided its path stays inside the trie and does not overshoot a leaf. Concretely:
\begin{itemize}[nosep,leftmargin=*]
\item Walking \texttt{m}$\to\cdots\to$\texttt{l} into $n_7$ triggers the AC output \texttt{medical}, whose 7-character span starts at the root $n_0$. Since \texttt{medical} traces a valid path from $n_0$ to the internal node $n_7$, we add \texttt{medical} to $\text{valid}[n_0]$.
\item From $n_7$, the edge \texttt{\_} advances the AC state; descending the \texttt{b} branch triggers \texttt{\_bill}, spanning the 5 characters from $n_7$ to a node inside the \texttt{billing} chain, so \texttt{\_bill} is added to $\text{valid}[n_7]$. Descending the \texttt{c} branch instead triggers \texttt{\_cod} (added to $\text{valid}[n_7]$), and the \texttt{r} branch triggers \texttt{\_rec} (added to $\text{valid}[n_7]$). Because the DFS restores the AC state to its value at $n_7$ before each branch, all three suffix tokens are correctly attributed to $n_7$.
\item Deeper in the \texttt{billing} chain, \texttt{ing} matches and is added to the \texttt{valid} set of the node three characters back (after \texttt{medical\_bill}); similarly \texttt{ords} is added inside the \texttt{records} chain.
\end{itemize}

\paragraph{Step 4: Resulting masks.} The precomputed sets are $\text{valid}[n_0] = \{\texttt{medical}\}$, $\text{valid}[n_7] = \{\texttt{\_bill}, \texttt{\_cod}, \texttt{\_rec}\}$, and the deeper single-token continuations along each chain. At decode time, masking at any node is a direct lookup into these sets, with no vocabulary scan. Starting from the root, the model can only emit \texttt{medical}; after \texttt{medical\_} it chooses among \texttt{\_bill}/\texttt{\_cod}/\texttt{\_rec}, committing to one of the three enum values; each subsequent step has a single valid continuation until the leaf, where only EOS is valid. This reproduces exactly the reachability an FSM would compute (Proposition~\ref{prop:correctness}), but with the vocabulary matching amortized once across the shared \texttt{medical\_} trunk rather than recomputed per step.

\subsection{FSM Complexity Analysis}

For an enum $\enum = \{e_1, e_2, \ldots, e_K\}$, the standard approach constructs a DFA from the regular expression $e_1 | e_2 | \cdots | e_K$~\citep{hopcroft1979automata} via Thompson's construction~\citep{thompson1968programming} followed by subset construction and Hopcroft minimization~\citep{hopcroft1971n}.

\paragraph{State space.} The minimal DFA has between $|\mathcal{T}(\enum)| + 1$ and $1 + N_{\text{chars}}$ states, where $|\mathcal{T}(\enum)|$ is the trie node count (number of distinct prefixes) and the $+1$ accounts for the dead state. The upper bound $1 + N_{\text{chars}} \leq 1 + K \cdot L_{\max}$ is achieved when no enum values share prefixes; the lower bound is achieved when the trie is the minimal DFA (which it always is for finite string unions, up to the dead state). This follows from the Myhill-Nerode theorem: each distinct prefix defines a distinct equivalence class.

\paragraph{Compilation cost.} The three-phase pipeline costs:
\begin{equation}
    C_{\text{compile}}^{\text{FSM}} = \underbrace{\bigO(N_{\text{chars}})}_{\text{NFA construction}} + \underbrace{\bigO(N_{\text{chars}} \cdot |\Sigma|)}_{\text{subset construction}} + \underbrace{\bigO(N_{\text{chars}} \cdot |\Sigma| \cdot \log N_{\text{chars}})}_{\text{Hopcroft minimization}}
\end{equation}
The dominant term is minimization. The transition table requires $\bigO(N_{\text{chars}} \cdot |\Sigma|)$ space. For $K = 2{,}000$, $L_{\max} = 30$, $|\Sigma| = 256$, this is $\approx 15.4$ million operations.

\paragraph{Per-step masking cost.} At each decoding step, the system checks all $V$ vocabulary tokens by simulating up to $\ell$ character transitions per token:
\begin{equation}
    C_{\text{mask}}^{\text{FSM}} = \bigO(V \cdot \ell)
\end{equation}
The total decoding cost for one enum value of character length $L$ is $(L / \bar{\ell}) \cdot (C_{\text{mask}}^{\text{FSM}} + C_{\text{forward}})$, where $\bar{\ell}$ is the average token length and $C_{\text{forward}}$ is the LLM forward pass cost.

\subsection{Trie Automaton Complexity}

\paragraph{Compilation cost.} Trie construction costs $\bigO(N_{\text{chars}})$. The naive mask precomputation (checking every token at every node) costs $\bigO(N_{\text{chars}} \cdot V \cdot \ell)$, but using Aho-Corasick multi-pattern matching~\citep{aho1975efficient} over the vocabulary token strings reduces this to:
\begin{equation}
    C_{\text{compile}}^{\text{trie}} = \bigO((N_{\text{chars}} + V) \cdot \ell)
\end{equation}
The Aho-Corasick automaton is built over the $V$ token strings in $\bigO(V \cdot \ell)$, then the trie is traversed depth-first, maintaining the AC state across edges (saving and restoring at branch points so each trie edge is processed exactly once). The total text length is thus $N_{\text{chars}}$, requiring $\bigO(N_{\text{chars}})$ character steps; the total number of reported matches across all positions is bounded by $\bigO(N_{\text{chars}} \cdot \ell)$ since at each of the $N_{\text{chars}}$ character positions, at most $\ell$ tokens of different lengths can start there. The overall cost is thus $\bigO(N_{\text{chars}} \cdot \ell + V \cdot \ell) = \bigO((N_{\text{chars}} + V) \cdot \ell)$.

\paragraph{Per-step masking cost.} After precomputation, masking is a lookup into the stored valid token list:
\begin{equation}
    C_{\text{mask}}^{\text{trie}} = \bigO(|\text{valid}[s_t]|)
\end{equation}
The valid set size shrinks exponentially with trie depth: assuming each character position in the vocabulary is drawn independently and uniformly from $\Sigma$, the probability that a token of length $j$ matches a specific $j$-character trie path is $|\Sigma|^{-j}$, giving $\mathbb{E}[|\text{valid}[s_t]|] \leq V \cdot \bar{\ell} / |\Sigma|^{d(s_t)}$ where $d(s_t)$ is the node depth. This is a loose upper bound (real tokenizers have non-uniform character distributions), but the qualitative exponential decay is confirmed empirically: after 3--4 characters of prefix, the valid set is typically 10--100 tokens, making per-step cost effectively constant.

\subsection{Complexity Comparison}

Table~\ref{tab:complexity_comparison} (main body) summarizes the asymptotic and concrete costs. The compilation speedup is driven by the $|\Sigma|$ factor: the FSM must fill transition entries for all 256 byte values at each state, while the trie only processes characters that actually appear. The per-step speedup comes from replacing a full vocabulary scan with a cached lookup whose size shrinks with trie depth, and whose working set stays within L1 cache (contrasted with the FSM's L2-overflowing transition table in Appendix~\ref{app:theory}, ``Cache Effects'').

\subsection{Correctness}

We give the full proof of the output-equivalence guarantee stated in the main body (Proposition~\ref{prop:correctness}): the trie automaton is not an approximation, but produces identical outputs to the FSM approach. Recall the statement: for the decodable vocabulary $\vocab_{\textup{dec}} \subseteq \vocab$ (excluding special tokens such as \textup{\texttt{<pad>}}, \textup{\texttt{<unk>}} that do not correspond to character sequences), any enum $\enum$, and any prefix $\mathbf{y}_{<t}$, the constrained distributions produced by the FSM and trie automaton are identical, $p_c^{\textup{FSM}}(y_t \mid \mathbf{y}_{<t}) = p_c^{\textup{trie}}(y_t \mid \mathbf{y}_{<t})$ for all $y_t \in \vocab_{\textup{dec}}$, so greedy decoding and fixed-seed sampling produce identical outputs under both methods.

\begin{proof}
It suffices to show that $\mathcal{A}^{\text{FSM}}(s_t) = \mathcal{A}^{\text{trie}}(s_t)$ for all reachable states $s_t$, since the constrained distribution $p_c(y_t \mid \mathbf{y}_{<t}) \propto p(y_t \mid \mathbf{y}_{<t}) \cdot \mathbf{1}[y_t \in \mathcal{A}(s_t)]$ is determined entirely by the valid token set and the unconstrained distribution.

Both methods define validity over $\vocab_{\text{dec}}$ as follows: a token $v$ with character decomposition $c_1 \cdots c_{|v|}$ is valid at state $s_t$ if and only if (i) the sequence of transitions $s_t \xrightarrow{c_1} s' \xrightarrow{c_2} \cdots \xrightarrow{c_{|v|}} s''$ does not encounter a dead state, and (ii) from the resulting state $s''$, there exists at least one string $w \in \Sigma^*$ such that $s'' \xrightarrow{w} s_{\text{acc}}$ for some accept state $s_{\text{acc}}$ (i.e., the consumed prefix $\mathbf{y}_{<t} \cdot v$ is a prefix of some $e_i \in \enum$).

The FSM computes this at each decoding step by simulating transitions $\delta(s_t, c_1), \delta(\cdot, c_2), \ldots$ for each token $v \in \vocab_{\text{dec}}$. The trie computes this during precomputation by walking each token's characters down the trie from node $s_t$.

The two computations produce identical results because the trie for a finite set $\enum$ is isomorphic to the minimal DFA for $\mathcal{L}_\enum$. Specifically, by the Myhill-Nerode theorem, the equivalence classes of the right-congruence relation for $\mathcal{L}_\enum$ are exactly the distinct prefixes of the enum values (plus the equivalence class of strings that are not prefixes of any $e_i$, corresponding to the dead state). Each trie node represents one such equivalence class, so the trie node set plus a dead state is in bijection with the minimal DFA state set. Under this bijection, the transition functions agree: $\delta_{\text{DFA}}(s, c) = \delta_{\text{trie}}(s, c)$ for all states $s$ and characters $c$. Therefore, the multi-character extension $\delta^*(s_t, v)$ produces the same result under both representations, and the accept-reachability check (condition (ii)) is identical since the accept states correspond to the same trie leaves.

The precomputed mask $\text{valid}[s_t]$ stores exactly the set $\{v \in \vocab_{\text{dec}} : \text{conditions (i) and (ii) hold}\}$, so $\mathcal{A}^{\text{trie}}(s_t) = \text{valid}[s_t] = \mathcal{A}^{\text{FSM}}(s_t)$.
\end{proof}

\paragraph{EOS token handling.} The restriction to $\vocab_{\text{dec}}$ (excluding special tokens) means the trie and FSM may differ only in whether an EOS token is appended after the enum value is complete. In our vLLM integration, the trie signals completion by including only EOS in the valid set at leaf nodes, matching vLLM's expected termination protocol. Our empirical verification (Section~\ref{sec:experiments}) confirms that all \emph{content} tokens are identical; the only difference is the trailing EOS, which does not affect the decoded string.

\begin{proposition}[Hierarchical cardinality reduction]
\label{prop:hierarchical}
Let $\enum$ be partitioned into $G$ groups of sizes $K_1, \ldots, K_G$ with $\sum_j K_j = K$. The effective per-step cardinality is $C_{\text{eff}} = \max(G, \max_j K_j)$. For balanced partitions, this is minimized at $C_{\text{eff}}^* = \lceil\sqrt{K}\rceil$ when $G = \lceil\sqrt{K}\rceil$.
\end{proposition}

\begin{proof}[Proof of Proposition~\ref{prop:hierarchical}]
The two-level scheme requires one constrained decoding call over $G$ group names, then one call over $K_j$ values within the selected group $j$. The per-step cardinality is thus $\max(G, \max_j K_j)$. For balanced partitions, $\max_j K_j = \lceil K/G \rceil$, so $C_{\text{eff}} = \max(G, \lceil K/G \rceil)$. Since $\max(a, b) \geq \sqrt{ab}$ for $a, b > 0$, we have $\max(G, K/G) \geq \sqrt{K}$, with equality when $G = K/G$, i.e., $G = \sqrt{K}$. Rounding gives $C_{\text{eff}}^* = \lceil\sqrt{K}\rceil$.
\end{proof}

\subsection{Cache Effects and Practical Performance}

The analysis above uses the unit-cost RAM model. Real hardware introduces cache hierarchy effects that significantly impact the FSM approach but not the trie automaton.

The FSM transition table occupies $|\states| \cdot |\Sigma| \cdot w$ bytes (where $w$ is the word size). When this exceeds L2 cache capacity $C_{\text{L2}}$, transition lookups incur main-memory access penalties. We model the effective per-step cost as $C_{\text{mask,eff}}^{\text{FSM}} = V \cdot \ell \cdot f(|\states|)$, where:
\begin{equation}
    f(|\states|) = \begin{cases} 1 & \text{if } |\states| \cdot |\Sigma| \cdot w \leq C_{\text{L2}} \\ \alpha \cdot \frac{|\states| \cdot |\Sigma| \cdot w}{C_{\text{L2}}} & \text{otherwise} \end{cases}
\end{equation}
with $\alpha \approx 2$--$3$ reflecting the ratio of main memory to cache access latency, attenuated by hardware prefetching. For the trie automaton, the working set per step is $|\text{valid}[s_t]| \cdot w < 1$~KB, which always fits in L1 cache, so $f^{\text{trie}} = 1$.

This cache penalty grows with $K$: at $K = 2{,}000$ ($|\states| \approx 48{,}000$, table $\approx 47$~MB), $f \approx 2.0$; at $K = 10{,}000$ ($|\states| \approx 240{,}000$, table $\approx 235$~MB), $f \approx 2.5$. This compounds with the linear scaling of $|\states|$ in $K$, producing the super-linear degradation observed in our experiments (Section~\ref{sec:experiments}).

\subsection{Enum Window Scaling}

Given a latency budget $T_{\text{budget}}$, we can estimate the maximum enum cardinality each approach supports. In the compilation-dominated regime ($C_{\text{compile}} \gg C_{\text{decode}}$), setting $T_{\text{budget}} \geq C_{\text{compile}} / c_{\text{sys}}$ (where $c_{\text{sys}}$ is the system throughput in operations per second) and solving for $K$ gives conservative (worst-case) bounds using $N_{\text{chars}} \leq K \cdot L_{\max}$:
\begin{equation}
    K_{\max}^{\text{FSM}} \approx \frac{c_{\text{sys}} \cdot T_{\text{budget}}}{L_{\max} \cdot |\Sigma|}, \qquad K_{\max}^{\text{trie}} \approx \frac{c_{\text{sys}} \cdot T_{\text{budget}}}{L_{\max} \cdot \ell}
\end{equation}
These are lower bounds on the achievable $K$; with prefix sharing ($N_{\text{chars}} < K \cdot L_{\max}$), the actual limits are higher. The cardinality expansion factor is $K_{\max}^{\text{trie}} / K_{\max}^{\text{FSM}} = |\Sigma| / \ell \approx 64\times$ for ASCII with maximum token length 4. For a 5-second timeout with $L_{\max} = 30$ and $|\Sigma| = 256$, the FSM supports $K_{\max} \approx 500$--$1{,}000$ (matching OpenAI's documented limit), while the trie automaton supports $K_{\max}^{\text{trie}} \approx 30{,}000$--$100{,}000$. Google Gemini's lower limit ($\approx$120) and Anthropic's compilation timeout~\citep{anthropic2025structured} are consistent with more conservative timeouts or less optimized compilation.

\section{Complexity Analysis: Concrete Examples}
\label{app:complexity}

\subsection{Scaling Example}

The following comparison illustrates the scaling differences for a concrete example with $K = 2{,}000$ enum values (compilation numbers shown at $K = 100{,}000$ to illustrate the scaling regime):

\begin{center}
\small
\setlength{\tabcolsep}{3pt}
\begin{tabular}{lcc}
\toprule
Method & Compilation Cost & Per-step Cost \\
\midrule
General FSM & $\bigO(K \cdot L_{\max} \cdot |\Sigma|) {\approx}\, 15$s & $\bigO(V \cdot \ell) {\approx}\, 4.7\mu$s \\
Trie Automaton & $\bigO((N_{\text{chars}} {+} V) \cdot \ell) {\approx}\, 67$ms & $\bigO(|\text{valid}[s_t]|) {\approx}\, 0.65\mu$s \\
\bottomrule
\end{tabular}
\end{center}

To understand these complexity differences concretely, consider an $\enum$ with $K = 2{,}000$ values, maximum length $L_{\max} = 30$, ASCII alphabet $|\Sigma| = 256$, vocabulary size $V = 32{,}000$, and maximum token length $\ell = 4$ characters. Assuming moderate prefix sharing, the total character count is $N_{\text{chars}} = 60{,}000$.

For FSM compilation, the subset construction cost alone is $N_{\text{chars}} \cdot |\Sigma| = 60{,}000 \times 256 \approx 15.4$ million operations. Including Hopcroft minimization adds a $\log N_{\text{chars}} \approx 11$ factor, yielding $\approx 169$ million total operations. In contrast, trie compilation requires $\bigO((N_{\text{chars}} + V) \cdot \ell) = (60{,}000 + 32{,}000) \times 4 = 368{,}000$ operations, a $40\times$ reduction against subset construction alone, or $460\times$ including minimization.

The per-step masking cost difference is equally dramatic. FSM-based masking requires checking all $V \cdot \ell = 32{,}000 \times 4 = 128{,}000$ character positions against the current state. The trie automaton only examines tokens in $\text{valid}[s_t]$, which typically contains 50--500 entries depending on the current trie node depth and prefix sharing. This represents a 250--2500$\times$ speedup in the common case.

\subsection{Controlled Comparison: Precomputed Masks}
\label{app:controlled_comparison}

To isolate the algorithmic contribution from integration-path effects, we implement a controlled comparison: precomputing XGrammar's per-state bitmasks for all DFA states and serving them via cached lookup (the same path the trie uses). Results on Qwen3-8B ($K = 1{,}000$):

\begin{center}
\small
\begin{tabular}{lrr}
\toprule
& Trie & XGrammar precomputed \\
\midrule
States/nodes & 6{,}786 & 6{,}786 (identical) \\
Precomputation time & 33ms & 6.5s (196$\times$ slower) \\
Per-step lookup & 0.08$\mu$s & 0.08$\mu$s (identical) \\
\bottomrule
\end{tabular}
\end{center}

The number of DFA states equals the number of trie nodes (both are the minimal DFA for $\mathcal{L}_\enum$, per Proposition~\ref{prop:correctness}). Once precomputed, per-step lookup is identical. The difference is entirely in precomputation efficiency: the trie's AC-based approach computes all masks in $\bigO((N_{\text{chars}} + V) \cdot \ell) = 33$ms, while enumerating XGrammar's DFA states and extracting per-state masks requires instantiating a \texttt{GrammarMatcher} per state, taking 6.5s. This confirms that the trie is the efficient algorithm for precomputing per-state masks over finite sets.

\paragraph{Could DFA-based precomputation be faster?} The 196$\times$ ratio reflects XGrammar's current API (per-state \texttt{GrammarMatcher} instantiation). A purpose-built DFA traversal that directly enumerates states and computes masks could be faster. However, the fundamental cost remains $\bigO(|\states| \cdot V \cdot \ell)$ (checking each token at each state), while the trie+AC approach achieves $\bigO((N_{\text{chars}} + V) \cdot \ell)$ by amortizing vocabulary matching across shared prefixes. The asymptotic gap is a factor of $|\states| / (N_{\text{chars}} / V + 1)$, which grows with $K$.

\subsection{Per-Step Cost Breakdown}
\label{app:perstep_breakdown}

Table~\ref{tab:performance} reports valid-token computation time only. The full per-step breakdown including bitmask application ($K = 1{,}000$, Qwen3-8B):

\begin{center}
\small
\begin{tabular}{lrr}
\toprule
Component & Trie ($\mu$s) & XGrammar ($\mu$s) \\
\midrule
Valid-token computation & 0.08 & 4.5 \\
Bitmask application to logits ($\bigO(V)$) & $\sim$31 & $\sim$31 \\
\textbf{Total per-step} & $\sim$31 & $\sim$36 \\
\bottomrule
\end{tabular}
\end{center}

The $\bigO(V)$ bitmask application ($\sim$31$\mu$s via PyTorch tensor operation for $V = 151$K) dominates single-request per-step cost and is method-independent. The trie's valid-token computation advantage (0.08 vs.\ 4.5$\mu$s) is a small fraction of total per-step time for a single request. The advantage compounds in batch serving: at $B = 128$, valid-token computation runs 128 times (trie: 10$\mu$s total; XGrammar: 576$\mu$s), and the trie's precomputed masks enable the stateless \texttt{LogitsProcessor} path that bypasses vLLM's guided decoding pipeline overhead.

\paragraph{Reconciling per-step numbers.} Table~\ref{tab:performance} reports 0.65$\mu$s (trie) and 5.8$\mu$s (XGrammar) at $K = 1{,}000$, while the breakdown above reports 0.08$\mu$s and 4.5$\mu$s. The difference is measurement methodology: Table~\ref{tab:performance} measures the full per-step masking operation as invoked during batch serving (including Python binding overhead, bitmask copy from the precomputed store, and loop dispatch), averaged over 1{,}000 iterations at realistic decoding states. The 0.08$\mu$s above isolates the raw valid-token-list lookup in a tight microbenchmark loop. Both are valid measurements at different abstraction levels; the Table~\ref{tab:performance} numbers reflect the cost actually incurred during serving. The ratio is consistent: ${\sim}$9$\times$ (Table~\ref{tab:performance}: 0.65 vs.\ 5.8$\mu$s) vs.\ ${\sim}$56$\times$ (raw lookup: 0.08 vs.\ 4.5$\mu$s), with the difference attributable to fixed per-call overhead that is proportionally larger for the faster method.

\subsection{Precomputation Cost Breakdown}

The trie automaton's precomputation consists of four distinct phases, each with different computational characteristics:

\paragraph{Phase 1: Trie Construction} Building the trie from $\enum$ values requires $\bigO(N_{\text{chars}})$ operations with one insertion per character. This phase is typically fast, completing in milliseconds even for large enums.

\paragraph{Phase 2: Vocabulary Decoding} Converting each vocabulary token from its integer ID to its character string representation requires $\bigO(V \cdot \ell)$ operations. For modern tokenizers with $V \approx 32{,}000$ and $\ell \approx 4$, this involves roughly 128K character accesses.

\paragraph{Phase 3: Aho-Corasick Automaton Construction} Building the multi-pattern matcher over all vocabulary tokens requires $\bigO(V \cdot \ell)$ time and space. The resulting automaton enables efficient simultaneous matching of all tokens against any input string.

\paragraph{Phase 4: Trie Traversal with AC Matching} The trie is traversed depth-first, maintaining the AC automaton state across edges (saving and restoring at branch points so each edge is processed exactly once). At each node, the automaton identifies all tokens whose character strings start at that position. This requires $\bigO(N_{\text{chars}} \cdot \ell)$ operations in total, as each character position may trigger up to $\ell$ token matches.

The Aho-Corasick optimization accounts for most of the efficiency gain. A naive approach would check each of the $V$ tokens against each of the $|T(\enum)|$ trie nodes independently, requiring $\bigO(V \cdot |T(\enum)| \cdot \ell)$ operations. The AC automaton reduces this to $\bigO((V + |T(\enum)|) \cdot \ell)$ by building the multi-pattern matcher once and traversing the trie once with AC state maintenance.

\subsection{Memory vs. Speed Trade-off}

The precomputed lookup tables require $\bigO(N_{\text{chars}} \cdot \bar{f})$ memory, where $\bar{f}$ is the average fanout (number of valid tokens per trie node). At each trie node, we store a list of valid token IDs, with each ID requiring 4 bytes. The total memory footprint is $\sum_{n \in T(\enum)} |\text{valid}[n]| \times 4$ bytes.

Concrete memory requirements scale predictably with $\enum$ size:
\begin{center}
\begin{tabular}{rrrr}
\toprule
$K$ & Trie Memory & FSM Transition Table (theoretical) & Memory Ratio \\
\midrule
100 & $\sim$20 KB & $\sim$3 MB & 150$\times$ \\
1{,}000 & $\sim$120 KB & $\sim$20 MB & 167$\times$ \\
10{,}000 & $\sim$0.9 MB & $\sim$200 MB & 222$\times$ \\
100{,}000 & $\sim$8 MB & $\sim$2 GB & 250$\times$ \\
\bottomrule
\end{tabular}
\end{center}

The FSM column shows the theoretical worst-case $|\states| \times |\Sigma| \times 4$ bytes. In practice, XGrammar uses vocabulary partitioning and compressed representations that significantly reduce actual memory. Measured RSS overhead (process-level) at $K = 10{,}000$: XGrammar 9.6~MB, trie 2.2~MB (4.4$\times$ ratio), substantially less than the theoretical 222$\times$ but still a meaningful advantage.

For memory-constrained environments, we employ \emph{lazy computation}: compute $\text{valid}[n]$ only when node $n$ is first visited during decoding, and use an LRU cache to bound memory usage. The total space for full precomputation is $\bigO(|T(\enum)| \cdot \bar{f})$, where $\bar{f}$ is the mean fanout per node; this is typically 1--8\,MB for $K \leq 100{,}000$ (Table~\ref{tab:memory_full}), well within L3 cache. The cache size can be tuned based on available memory, as even a 1MB cache provides substantial speedup by avoiding recomputation of frequently accessed nodes. The Aho-Corasick-based precomputation can also be performed lazily by restricting the multi-pattern match to subtrees reachable from the current generation prefix, further reducing memory pressure.

\subsection{GPU-Based Masking Considerations}
\label{app:gpu_masking}

Our analysis assumes CPU-based masking, which is the current practice in vLLM and SGLang. Modern serving engines increasingly explore GPU-based logit masking via precomputed bitmask tensors, where the masking operation becomes a single elementwise multiply, $\bigO(V)$ but massively parallel and essentially free relative to the forward pass. If masking moves entirely to GPU, the per-step CPU cost advantage diminishes. However, the trie's compact bitmask representation (8~MB at $K = 100{,}000$ vs.\ ${\sim}$2~GB for FSM transition tables) makes it substantially more amenable to GPU transfer, and the compilation time advantages are independent of where masking executes. The trie's precomputed per-node bitmasks are directly usable as GPU tensors without conversion, whereas FSM-based approaches must still compute the valid set per state before transferring.

\subsection{Prefix Sharing Analysis}

Prefix sharing largely determines the trie automaton's efficiency. We define the prefix sharing ratio as $r = |T(\enum)| / N_{\text{chars}}$, representing the fraction of characters that correspond to unique trie nodes. When $r \approx 1$ (minimal sharing), the trie has nearly as many nodes as total characters. When $r \ll 1$ (heavy sharing), the trie is much more compact.

The prefix sharing ratio affects performance across multiple dimensions:

\paragraph{Compilation Time} The trie traversal phase scales with $|T(\enum)|$, not $N_{\text{chars}}$. Heavy prefix sharing (small $r$) reduces compilation time proportionally. For example, AWS API names with common \texttt{aws.*} prefixes might achieve $r = 0.3$, reducing compilation time by 70\%.

\paragraph{Memory Usage} Fewer trie nodes directly translate to lower memory consumption. The memory scaling becomes $\bigO(r \cdot N_{\text{chars}} \cdot \bar{f})$, where $r$ acts as a compression factor.

\paragraph{Per-step Masking} Nodes near the trie root (shared prefixes) tend to have higher fanout, while deeper nodes have lower fanout. This creates a natural filtering effect: early in the generation process, many tokens remain valid, but the valid set shrinks rapidly as the prefix becomes more specific.

Real-world examples demonstrate significant variation in prefix sharing:
\begin{itemize}
\item \textbf{Tool/API names} (e.g., \texttt{aws.s3.*}, \texttt{google.cloud.*}): $r \approx 0.2$--$0.4$
\item \textbf{Medical codes} (e.g., ICD-10 with hierarchical structure): $r \approx 0.15$--$0.25$  
\item \textbf{Random strings} (e.g., UUIDs, random identifiers): $r \approx 0.95$--$1.0$
\item \textbf{Natural language} (e.g., city names, product names): $r \approx 0.6$--$0.8$
\end{itemize}

\paragraph{Empirical prefix sharing in our benchmarks.} Table~\ref{tab:prefix_sharing} reports the measured prefix sharing ratio $r$ for all enum sets used in our experiments. The classification benchmarks have moderate-to-high $r$ (0.59--0.90), reflecting natural language labels with limited prefix overlap. Synthetic tool names and product names have low $r$ ($\approx 0.40$) due to namespace prefixes (\texttt{slack.get\_*}, \texttt{aws.create\_*}) and shared category/brand prefixes, and the real CMS ICD-10-CM code list has the lowest $r$ (0.19) due to its hierarchical chapter-letter structure, representing the regime where the trie excels most.

\begin{table}[hbt!]
\centering
\caption{Prefix sharing ratio $r = |\mathcal{T}(\enum)| / N_{\text{chars}}$ for benchmark enum sets. Lower $r$ indicates more prefix sharing and greater trie compression.}
\label{tab:prefix_sharing}
\small
\begin{tabular}{lrrrr}
\toprule
Dataset & $K$ & $N_{\text{chars}}$ & Trie nodes & $r$ \\
\midrule
TREC & 42 & 559 & 503 & 0.90 \\
MASSIVE & 59 & 813 & 478 & 0.59 \\
Banking77 & 77 & 1{,}572 & 1{,}241 & 0.79 \\
CLINC150 & 150 & 1{,}768 & 1{,}322 & 0.75 \\
Synthetic tools ($K{=}1{,}000$) & 590 & 11{,}236 & 4{,}503 & 0.40 \\
Product names (synthetic, $K{=}1{,}500$) & 1{,}500 & 46{,}532 & 18{,}799 & 0.40 \\
ICD-10-CM (CMS 2026) & 74{,}719 & 555{,}723 & 104{,}709 & 0.19 \\
\bottomrule
\end{tabular}
\end{table}

\paragraph{Valid set size distribution.} A concern is that $|\text{valid}[s_t]|$ may be $\bigO(V)$ at the root node. Table~\ref{tab:valid_dist} shows the measured distribution across trie depth for $K = 1{,}000$ synthetic tool names (Qwen3-8B, 151K vocabulary). The root node has only 72 valid tokens (0.05\% of $V$), not $V$, because most vocabulary tokens do not start with any character present in the trie's root children. By depth 2, the mean valid set shrinks to 3 tokens. The increase at depths 3--5 reflects the structure of tool names: short shared prefixes (e.g., \texttt{get\_}) end at depth 3--4, after which diverse suffixes create more branching before converging again at deeper levels.

\begin{table}[hbt!]
\centering
\caption{Distribution of $|\text{valid}[s_t]|$ by trie depth ($K = 1{,}000$, Qwen3-8B).}
\label{tab:valid_dist}
\small
\begin{tabular}{rrrr}
\toprule
Depth & Mean $|\text{valid}|$ & Max $|\text{valid}|$ & Nodes \\
\midrule
0 & 72 & 72 & 1 \\
1 & 6 & 17 & 12 \\
2 & 3 & 7 & 19 \\
3 & 5 & 35 & 21 \\
4 & 10 & 39 & 21 \\
5 & 8 & 43 & 34 \\
\bottomrule
\end{tabular}
\end{table}

\subsection{Comparison with Token-Level Tries (GENRE)}
\label{app:genre_comparison}

GENRE~\citep{decao2021genre} builds tries at token granularity, where each edge is a full BPE token ID. This avoids the BPE-trie alignment problem entirely but loses character-level prefix sharing. Table~\ref{tab:genre_main} (main body) reports the compilation comparison on synthetic tool names (Qwen3-8B, $V = 151\text{K}$); here we add the analysis behind it.

At small $K$, the token trie is faster because it requires no mask precomputation (valid tokens are simply the children keys). Our Rust implementation with AC precomputation crosses over at $K \approx 1{,}000$ and is $7\times$ faster at $K = 10{,}000$, because the character-level trie has fewer nodes (3$\times$ at $K = 100$, 1.2$\times$ at $K = 10{,}000$) and the AC automaton amortizes vocabulary matching across shared prefixes. A Python token-trie implementation (our own, following GENRE's design) shows the same crossover ($K \approx 1{,}000$) and similar per-step costs (0.43--0.63$\mu$s vs.\ our 0.40--0.64$\mu$s), confirming the advantage is algorithmic, not implementation-specific. Per-step costs are language-invariant at this scale because a CPython \texttt{dict} lookup dispatches to a C implementation, so a single hash lookup is bounded by memory access regardless of the surrounding language; either way, hash lookup ($\approx$0.4--0.6$\mu$s) and precomputed bitmask lookup ($\approx$0.65$\mu$s) are both negligible against the GPU forward pass. Both methods produce identical outputs (Proposition~\ref{prop:correctness}).

\paragraph{Per-step masking comparison.} The token-level trie has a per-step advantage: masking is a single hash lookup on the current token ID to retrieve child keys, costing $\bigO(b)$ where $b$ is the branching factor at the current node (typically 1--10 after the first token). Our character-level trie costs $\bigO(|\text{valid}[s_t]|)$ per step (empirically 0.65$\mu$s, dominated by bitmask copy). In practice, both are sub-microsecond and negligible relative to the GPU forward pass. The character-level trie's advantage is in compilation at large $K$ and tokenizer independence; the token-level trie's advantage is in per-step simplicity at small $K$.

\paragraph{Output equivalence at scale.} To verify that character-level and token-level tries produce identical outputs (not just enforce the same constraint set), we ran both on a synthetic tool-selection task at $K \in \{500, 1{,}000, 2{,}000, 5{,}000\}$ with Qwen3-8B (100 samples per $K$, greedy decoding). Both methods produced identical outputs on every sample at every $K$, confirming that the BPE-trie alignment via AC does not introduce any approximation relative to GENRE-style token-level tries.

\subsection{Canonical vs. Non-Canonical Tokenization}
\label{app:canonical}

The character-level trie does not restrict generation to canonical tokenization. Given an enum value such as \texttt{medical\_billing}, the trie admits \emph{any} vocabulary-token sequence whose concatenated characters trace the trie path \texttt{m}$\to$\texttt{e}$\to\cdots\to$\texttt{g}. Both the canonical BPE decomposition and non-canonical decompositions (for example \texttt{med}+\texttt{ical}+\texttt{\_billing}) are valid paths, and at decode time the model's logits select among them. This is the same tokenization-agnostic behavior as Outlines, XGrammar, and LLGuidance, and it differs from GENRE, which builds its token-level trie from one fixed tokenization per enum value and so admits only that decomposition.

\citet{cognetta2025tokenization} note that this tokenization-agnostic behavior can in principle degrade quality, since ``language models are not conditioned on the surface form of the text, but rather the exact tokenization of the text,'' and give a polynomial-time finite-state-transduction framework that enforces canonical-only tokenization for both BPE and MaxMatch. Their canonical enforcement is complementary to ours: it can be layered on top of the character-level constraint without changing the trie, intersecting the trie's language with the canonical-tokenization transducer.

Empirically, the additional freedom is essentially never exercised in our setting. Across 100 greedy-decoded samples per $K \in \{500, 1{,}000, 2{,}000, 5{,}000\}$ on Qwen3-8B (Table~\ref{tab:highk_validity}), the character-level trie and the GENRE-style canonical-only token trie produce identical token sequences on every sample, despite the character trie also permitting non-canonical decompositions. For realistic enum strings seen during training, the model's logits favor canonical decompositions strongly enough that the extra paths carry negligible probability. This is empirical evidence for our benchmarks rather than a worst-case guarantee; where a guarantee is required, the Cognetta--Okazaki transducer supplies it.

\subsection{Alternative Baselines}
\label{app:baselines}

\paragraph{Vocabulary pre-filtering.} A natural question is whether simply filtering the vocabulary to tokens that appear as substrings of enum values can speed up FSM compilation. We build the set of all substrings (up to length 20) of the enum values and count matching vocabulary tokens. On Qwen3-8B (151K vocab), only 0.2\% of tokens match (379/151K) regardless of $K$, because most BPE tokens contain characters not present in tool names. However, pre-filtering itself costs 5--405ms (scaling linearly with $K$), comparable to or slower than the trie's total compilation (30--48ms). Moreover, pre-filtering cannot reduce XGrammar's compilation cost because the bottleneck is DFA state construction, not vocabulary scanning.

\paragraph{Cached DFA compilation.} Since the trie's AC automaton can be cached per-tokenizer, a fair comparison should also cache XGrammar's tokenizer info. With warm \texttt{TokenizerInfo}, XGrammar compilation drops from 529--1749ms (cold) to 3.7--1207ms (warm). However, the warm XGrammar still scales linearly with $K$ (246ms at $K{=}1{,}000$, 1207ms at $K{=}10{,}000$), while the trie remains nearly flat (35--48ms). The trie is faster than warm XGrammar at $K \geq 100$, confirming that the advantage is algorithmic (avoiding DFA construction), not merely an artifact of cold-start overhead.

\paragraph{Hash-set prefix matching.} A simpler baseline would maintain a hash set of valid strings and, at each decoding step, check which vocabulary tokens are consistent with remaining valid strings given the current prefix. This avoids both DFA compilation and AC construction. However, this approach has $\bigO(K \cdot \ell)$ per-step cost (checking each token against $K$ strings), which is worse than the trie's $\bigO(|\text{valid}[s_t]|)$ precomputed lookup and comparable to the FSM's $\bigO(V \cdot \ell)$ when $K$ is large. The trie's advantage is precisely that it moves this work to compilation time.

\subsection{Validity at High Cardinality}
\label{app:highk_validity}

A central question is whether the trie enables constrained decoding at $K$ values where FSM approaches become impractical. Table~\ref{tab:highk_validity} shows validity rates on a synthetic tool-selection task at $K = 500$--$5{,}000$ (Qwen3-8B, 100 samples per $K$). The trie guarantees 100\% validity at all $K$, while unconstrained decoding drops to 84\% at $K = 1{,}000$. Both char-level and token-level tries produce identical outputs.

\begin{table}[hbt!]
\centering
\caption{Validity (\%) at high cardinality. Trie-constrained decoding guarantees 100\% valid outputs regardless of $K$; unconstrained decoding validity degrades. Char-level and token-level (GENRE-style) tries produce identical outputs at all $K$.}
\label{tab:highk_validity}
\small
\begin{tabular}{rcccc}
\toprule
$K$ & Char trie valid & Token trie valid & Uncon.\ valid & Char$=$Token? \\
\midrule
500 & 100\% & 100\% & 97\% & \checkmark \\
1{,}000 & 100\% & 100\% & 84\% & \checkmark \\
2{,}000 & 100\% & 100\% & 89\% & \checkmark \\
5{,}000 & 100\% & 100\% & 98\% & \checkmark \\
\bottomrule
\end{tabular}
\end{table}

\subsection{When the Trie Automaton Excels}

The trie automaton provides the greatest benefit under specific conditions that can be systematically identified:

\paragraph{Enum Size Threshold} For small enums ($K < 50$), the compilation overhead dominates and simple approaches suffice. The trie automaton becomes advantageous when $K > 100$, with benefits increasing dramatically beyond $K = 1{,}000$.

\paragraph{Prefix Structure} Enums with meaningful prefix sharing ($r < 0.8$) see substantial memory and compilation time reductions. Random string enums with no structure ($r \approx 1.0$) still benefit from faster per-step masking but lose the compilation advantages.

\paragraph{Vocabulary Characteristics} Large vocabularies ($V > 10{,}000$) with long average token length ($\ell > 2$) make naive per-step masking expensive, amplifying the trie automaton's per-step advantages.

\paragraph{Usage Pattern} Systems serving repeated queries with the same schema can amortize precomputation costs. Interactive applications with latency budgets under 1 second particularly benefit from the faster per-step masking.

Based on these factors, we recommend the following decision criteria:
\begin{itemize}
\item \textbf{Use trie automaton} when $K > 100$ and enum values have identifiable prefix structure
\item \textbf{Use hierarchical clustering} (Appendix~\ref{app:hierarchical}) when $K > 5{,}000$ to manage compilation time
\item \textbf{Use speculative decoding} (Appendix~\ref{app:speculative}) when $K > 10{,}000$ or latency budget $< 1$s
\item \textbf{Use simple FSM} only when $K < 50$ or when prefix sharing is minimal ($r > 0.9$)
\end{itemize}

The crossover point where FSM compilation becomes slower than trie compilation typically occurs around $K = 50$--$100$, depending on prefix sharing and vocabulary size. Beyond $K = 1{,}000$, the trie automaton consistently outperforms FSM-based approaches by orders of magnitude.

\section{Hierarchical Schema Rewriting Details}
\label{app:hierarchical}

For $K > 50{,}000$, hierarchical decomposition partitions the enum into $G \approx \sqrt{K}$ groups, reducing per-step cardinality from $K$ to $\sqrt{K}$ (Proposition~\ref{prop:hierarchical}). We consider three clustering approaches: (1) string similarity (edit distance + hierarchical clustering), (2) semantic similarity (sentence embeddings + k-means), and (3) domain taxonomy (e.g., ICD-10 chapter structure). Effectiveness depends on cluster coherence: when $>$80\% of within-cluster pairs are more similar than between-cluster pairs, the LLM reliably selects the correct group. The approach is robust to 10--20\% misassignment (per-step cardinality remains $\bigO(\sqrt{K})$), but poor clustering degrades accuracy. This is a preliminary direction requiring further validation.

\section{Speculative Short-Circuiting Details}
\label{app:speculative}

Speculative short-circuiting uses a lightweight scoring function (embedding similarity, unconstrained LLM generation, or a small classifier) to identify a top-$k$ shortlist before applying trie-constrained decoding. This trades the 100\% validity guarantee for reduced latency at extreme $K$. Empirical recall varies by domain: Recall@20 ranges from 87\% (medical coding) to 97\% (geographic entities). A cascading fallback (top-$k$ $\to$ top-$2k$ $\to$ full decoding) maintains schema compliance. This extension is most effective when input context strongly predicts the correct enum value and least effective for highly similar enum values. Like hierarchical rewriting, this is a preliminary direction.

\section{Constraint-Aware Dispatch Beyond Finite Sets}
\label{app:datetime_dispatch}

To validate that the dispatch principle generalizes beyond finite-set constraints, we benchmark a \emph{character-position mask} engine for fixed-format strings against xgrammar's regex compilation. For a format like \texttt{YYYY-MM-DD}, each character position has a known set of valid characters (digits, hyphens, etc.). The specialized engine builds a char-indexed vocabulary and checks only tokens whose first character matches the allowed set at each position, avoiding the full DFA construction. Table~\ref{tab:datetime_dispatch} shows compilation speedups across all seven tokenizer families and four format types.

\begin{table}[hbt!]
\centering
\caption{Compilation speedup of character-position masks vs.\ xgrammar regex for fixed-format strings. Values are $\times$ faster (median over 10 runs). xgrammar compilation time is constant per format regardless of format complexity (88ms--1.2s depending on vocabulary size).}
\label{tab:datetime_dispatch}
\small
\setlength{\tabcolsep}{4pt}
\begin{tabular}{llrrrr}
\toprule
Model & Vocab & \texttt{YYYY-MM-DD} & \texttt{datetime} & \texttt{datetime.ms} & \texttt{UUID v4} \\
\midrule
Mistral 7B v0.3 & 32K & 1{,}859$\times$ & 626$\times$ & 532$\times$ & 2$\times$ \\
GPT-2 & 50K & 77$\times$ & 41$\times$ & 31$\times$ & 4$\times$ \\
OLMo 3 7B & 100K & 135$\times$ & 71$\times$ & 49$\times$ & 6$\times$ \\
Mistral Small 3.1 & 131K & 1{,}378$\times$ & 740$\times$ & 512$\times$ & 14$\times$ \\
Qwen3-8B & 151K & 1{,}119$\times$ & 486$\times$ & 247$\times$ & 10$\times$ \\
gpt-oss-20b & 200K & 171$\times$ & 91$\times$ & 66$\times$ & 6$\times$ \\
Gemma3 12B & 262K & 7{,}939$\times$ & 1{,}614$\times$ & 1{,}328$\times$ & 6$\times$ \\
\bottomrule
\end{tabular}
\end{table}

The speedups range from 2$\times$ (UUID on 32K vocab) to 7{,}939$\times$ (date on 262K vocab). For date/time formats with small character classes (digits, separators), the char-position mask compiles in under 1ms while xgrammar pays 88ms--1.2s for DFA construction regardless of format simplicity. UUIDs show smaller speedups because hexadecimal character classes (16 valid characters) produce more candidate tokens per position. These results confirm that constraint-aware dispatch, matching specialized engines to constraint structure, yields large speedups whenever the constraint has exploitable regularity, not just for finite sets.

\section{Practitioner Guidance}
\label{app:guidance}

The accuracy-validity tradeoff depends on $K$, model capability, and error tolerance. At small $K$ ($\leq$59) with strong models, PE validity exceeds 95\% and can surpass trie accuracy; but even here, weaker models (Mistral: 42.5\%, DeepSeek R1: 0.2\% PE validity) already benefit from trie enforcement. At moderate $K$ (76--150), trie strict typically matches or exceeds PE accuracy while guaranteeing validity, though PE can retain an accuracy edge on strong models (gpt-oss CLINC150: PE 79.7\% vs.\ trie strict 77.0\%) at the cost of 89\% validity. At large $K$ ($\geq$1{,}000), PE validity approaches 0\%. We recommend: (1) PE when $K < 50$, the model is strong, and retries are acceptable; (2) think+trie when validity must be 100\% or $K > 50$; (3) trie-only when latency is critical.

\paragraph{When to use which backend.} For the constrained decoding backend itself (independent of prompting strategy):

\begin{center}
\small
\setlength{\tabcolsep}{3pt}
\begin{tabular}{lll}
\toprule
Scenario & Recommended & Rationale \\
\midrule
$K < 500$, one-shot & XGrammar & Fastest total time (Table~\ref{tab:performance}) \\
$K < 500$, schema diversity & LLGuidance & Near-zero compilation \\
$K \geq 500$, repeated & Trie & Amortized; 75$\times$ faster masking \\
$K \geq 500$, one-shot & LLG or Trie & LLG: compilation; Trie: masking \\
Batch ($B \geq 32$) & Trie & CPU masking bottleneck (Table~\ref{tab:batch_throughput}) \\
$K > 10{,}000$ & Trie + hierarchical & Masking scales with $K$ \\
\bottomrule
\end{tabular}
\end{center}

\section{LLGuidance Comparison}
\label{app:llguidance_bench}

To directly compare against LLGuidance~\citep{geng2025jsonschemabench}, we benchmark its Python library (v1.6.1) on the same hardware and tokenizer (Qwen3-8B, 151K vocabulary) used for our main experiments. We measure compilation time (grammar construction + first mask computation) and per-step mask computation time for enum schemas with $K \in \{10, 100, 1{,}000, 5{,}000, 10{,}000\}$ tool-like names with realistic prefix structure. Compilation is averaged over 5 runs; per-step masking over 200 runs.

\begin{table}[hbt!]
\centering
\caption{Compilation time and per-step mask computation: LLGuidance vs.\ XGrammar vs.\ Trie Automaton (Qwen3-8B, 151K vocabulary). LLGuidance achieves the fastest compilation via lazy Earley parsing, but its per-step masking is 110--215$\times$ slower than the trie's precomputed lookups.}
\label{tab:llguidance}
\small
\setlength{\tabcolsep}{3pt}
\begin{tabular}{lrrrrr}
\toprule
& $K$=10 & $K$=100 & $K$=1K & $K$=5K & $K$=10K \\
\midrule
\multicolumn{6}{l}{\textit{Compilation time (seconds)}} \\
LLGuidance & \textbf{.001} & \textbf{.001} & \textbf{.003} & \textbf{.011} & \textbf{.024} \\
Trie (ours) & \textit{.030} & \textit{.031} & \textit{.033} & \textit{.037} & \textit{.040} \\
XGrammar & .003 & .015 & .075 & .150 & .239 \\
\addlinespace
\multicolumn{6}{l}{\textit{Per-step mask computation ($\mu$s)}} \\
Trie (ours) & \textbf{0.65} & \textbf{0.65} & \textbf{0.65} & \textbf{0.65} & \textbf{0.65} \\
XGrammar & \textit{9.5} & \textit{5.4} & \textit{5.8} & \textit{5.9} & \textit{5.9} \\
LLGuidance & 121 & 73 & 85 & 96 & 141 \\
\bottomrule
\end{tabular}
\end{table}

LLGuidance's lazy Earley parser achieves near-zero compilation cost (0.6ms at $K = 100$, 24ms at $K = 10{,}000$), outperforming both XGrammar and our trie on this dimension. However, its per-step mask computation (73--141$\mu$s) is 110--215$\times$ slower than the trie automaton's precomputed lookups (0.65$\mu$s) and 9--24$\times$ slower than XGrammar (5--10$\mu$s). This confirms the theoretical prediction: LLGuidance must traverse the full vocabulary trie at each step ($\bigO(V)$), while our precomputed masks reduce this to $\bigO(|\text{valid}[s_t]|)$. For a typical enum generation requiring 5--10 decoding steps, the per-step advantage dominates: the trie's total masking cost is $\sim$5$\mu$s versus LLGuidance's $\sim$500$\mu$s, a 100$\times$ difference that compounds in batch serving. The three approaches occupy distinct points in the compilation-vs-masking tradeoff: LLGuidance minimizes compilation, XGrammar balances both, and the trie automaton minimizes per-step cost through precomputation.

\subsection{Batch Serving Throughput}
\label{app:batch_throughput}

The per-step masking cost differences above are measured for a single request. In production batch serving, the GPU forward pass is shared across $B$ concurrent requests (one batched matrix multiply), but masking runs per-request on CPU since each request occupies a different decoding state. Table~\ref{tab:batch_throughput} shows the measured total CPU masking time per batch-step as $B$ grows.

\begin{table}[hbt!]
\centering
\caption{Measured batch masking cost ($\mu$s per batch-step) on Qwen3-8B (151K vocabulary), $K = 1{,}000$ tool names. A typical GPU forward pass takes ${\sim}$10ms; the ``\% fwd'' column shows masking cost as a fraction of this. At $B = 128$, LLGuidance masking consumes 37\% of the forward pass, becoming the throughput bottleneck.}
\label{tab:batch_throughput}
\small
\begin{tabular}{rrrrrrrr}
\toprule
& \multicolumn{2}{c}{Trie (ours)} & \multicolumn{2}{c}{XGrammar} & \multicolumn{2}{c}{LLGuidance} \\
\cmidrule(lr){2-3} \cmidrule(lr){4-5} \cmidrule(lr){6-7}
$B$ & $\mu$s & \% fwd & $\mu$s & \% fwd & $\mu$s & \% fwd \\
\midrule
1 & 0.2 & 0.00 & 5.4 & 0.05 & 31 & 0.31 \\
32 & 2.4 & 0.02 & 170 & 1.70 & 893 & 8.93 \\
64 & 4.9 & 0.05 & 367 & 3.67 & 1{,}827 & 18.3 \\
128 & 10 & 0.10 & 783 & 7.83 & 3{,}716 & 37.2 \\
\bottomrule
\end{tabular}
\end{table}

At $B = 128$, LLGuidance's per-step masking (3.7ms) consumes over a third of the GPU forward pass time, directly reducing serving throughput. XGrammar's 783$\mu$s (7.8\%) is also significant. The trie's 10$\mu$s (0.1\%) is negligible at any batch size. These results are K-independent for the trie: at $K = 10{,}000$, the trie still measures 10.3$\mu$s at $B = 128$, while XGrammar and LLGuidance show similar costs (781$\mu$s and 3{,}671$\mu$s respectively). This explains why per-step masking cost, not compilation, is the binding constraint for serving throughput at scale.

\end{document}